\documentclass{article}

\usepackage[final]{opt_template}
\usepackage{microtype}
\usepackage{graphicx}
\usepackage{subfigure}
\usepackage{booktabs} 

\usepackage{multirow}
\usepackage{xcolor}
\usepackage{mathtools}
\usepackage{relsize}
\usepackage{url}
\usepackage{nicefrac}
\usepackage{xspace}
\usepackage{algorithm}
\usepackage{algorithmic}
\usepackage{bbm}

\usepackage{fontenc}

\usepackage{amsfonts}
\usepackage{amsmath}
\usepackage{amsthm}
\usepackage{amssymb} 

\usepackage{mdframed} 
\usepackage{thmtools}

\definecolor{shadecolor}{gray}{0.95}
\declaretheoremstyle[
headfont=\normalfont\bfseries,
notefont=\mdseries, notebraces={(}{)},
bodyfont=\normalfont,
postheadspace=0.5em,
spaceabove=1pt,
mdframed={
  skipabove=8pt,
  skipbelow=8pt,
  hidealllines=true,
  backgroundcolor={shadecolor},
  innerleftmargin=4pt,
  innerrightmargin=4pt}
]{shaded}

 \usepackage{xcolor}
\usepackage{color}
\usepackage{graphicx}

\usepackage{algorithm}
\usepackage{verbatim}
\usepackage[algo2e]{algorithm2e} 

\newcommand{\R}{\mathbb{R}} 

\newcommand{\cO}{{\cal O}}

\newcommand{\cR}{{\cal R}}

\usepackage[colorinlistoftodos,bordercolor=orange,backgroundcolor=orange!20,linecolor=orange,textsize=scriptsize]{todonotes}

\newcommand{\eqdef}{\overset{\triangle}{=}} 
\newcommand{\norm}[1]{\left\lVert#1\right\rVert}      

\DeclareMathOperator{\Cov}{Cov}         
\DeclareMathOperator{\Var}{Var}         
\DeclareMathOperator{\Prob}{Prob}

\DeclarePairedDelimiter\ceil{\lceil}{\rceil}

\newcommand{\E}[1]{{\rm E}\left[#1\right] } 
\newcommand{\EE}[2]{{\rm E}_{#1}\left[#2\right] }

\theoremstyle{plain}
\newtheorem{theorem}{Theorem}  
\newtheorem{lemma}[theorem]{Lemma} 

\theoremstyle{definition}
\newtheorem{assumption}{Assumption} 
\newtheorem{definition}{Definition}

\theoremstyle{remark}
\newtheorem{remark}{Remark} 

\usepackage{hyperref}

\newcommand{\reals}{\mathbb{R}}

\newcommand{\jk}[1]{#1^{(j)}_k}
\newcommand{\jkp}[1]{#1^{(j)}_{k+1}}

\newcommand{\jkn}[1]{#1^{(j)}_{0}}
\newcommand{\tx}{\tilde{x}}
\newcommand{\lb}{\left(}
\newcommand{\rb}{\right)}

\newcommand{\td}{\tilde}
\newcommand{\eps}{\epsilon}
\renewcommand{\E}{\mathbb{E}}
\renewcommand{\R}{\mathbb{R}}
\renewcommand{\cR}{\mathcal{R}}
\renewcommand{\P}{\mathbb{P}}

\newcommand{\Ll}{\mathcal{L}}
\newcommand{\la}{\left\langle}
\newcommand{\ra}{\right\rangle}

\newcommand{\es}{\mathrm{es}}
\newcommand{\comp}{\mathrm{Comp}}
\renewcommand{\O}{\mathcal{O}}
\newcommand{\xj}{x^{(j)}}

\newcommand{\nuj}{\nu^{(j)}}

\newcommand{\Mj}{M^{(j)}}

\newcommand{\sI}{\mathcal{I}^{(j)}}
\newcommand{\etaj}{\eta_{j}}

\newcommand{\chij}{\chi_{j}}
\newcommand{\lamt}{\lambda_{t}}
\newcommand{\lamj}{\lambda_{j}}
\newcommand{\lamjj}{\lambda_{j-1}}
\newcommand{\DeltaL}{\Delta_{L}}
\newcommand{\tDeltaL}{\td{\Delta}_{L}}

\newcommand{\T}{T}
\newcommand{\Bj}{B_{j}}

\newcommand{\Bt}{B_{t}}
\newcommand{\tBj}{\td{B}_{j}}
\newcommand{\tBjj}{\td{B}_{j-1}}
\newcommand{\tBt}{\td{B}_{t}}
\newcommand{\bj}{b_{j}}

\newcommand{\mj}{m_{j}}

\newcommand{\Nj}{N_{j}}
\newcommand{\Lj}{\mathcal{L}_{j}}
\newcommand{\Ljj}{\mathcal{L}_{j-1}}

\makeatletter
\newcommand{\printfnsymbol}[1]{%
  \textsuperscript{\@fnsymbol{#1}}%
}

\allowdisplaybreaks
\title{Adaptivity of Stochastic Gradient Methods for Nonconvex
Optimization}

\author{
  Samuel Horv\'{a}th\thanks{equal contribution}\\
  Visual Computing Center\\
  KAUST, Thuwal\\
  Saudi Arabia \\
  \texttt{samuel.horvath@kaust.edu.sa} \\
  \And
  Lihua Lei\printfnsymbol{1}\\
  Department of Statistics\\
  Stanford University\\
  Palo Alto, CA 94305 \\
  \texttt{lihualei@stanford.edu} \\
  \And
  Peter Richt\'{a}rik \\
  Visual Computing Center\\
  KAUST, Thuwal\\
  Saudi Arabia \\
  \texttt{peter.richtarik@kaust.edu.sa} \\
  \And
  Michael I. Jordan \\
  Computer Science Division \& Department of Statistics \\
  University of California, Berkeley\\
  Berkeley, CA 94704\\
  \texttt{jordan@cs.berkeley.edu} \\
  }

\begin{document}

\maketitle

\begin{abstract}
Adaptivity is an important yet under-studied property in modern optimization theory. The gap between the state-of-the-art theory and the current practice is striking in that algorithms with desirable theoretical guarantees typically involve drastically different settings of hyperparameters, such as step-size schemes and batch sizes, in different regimes. Despite the appealing theoretical results, such divisive strategies provide little, if any, insight to practitioners to select algorithms that work broadly without tweaking the hyperparameters. In this work, blending the ``geometrization'' technique introduced by \citet{lei2016less} and the \texttt{SARAH} algorithm of \citet{nguyen2017sarah}, we propose the Geometrized \texttt{SARAH} algorithm for non-convex finite-sum and stochastic optimization. Our algorithm is proved to achieve adaptivity to both the magnitude of the target accuracy and the Polyak-\L{}ojasiewicz (PL) constant, if present. In addition, it achieves the best-available convergence rate for non-PL objectives simultaneously while outperforming existing algorithms for PL objectives. 
\end{abstract}

\section{Introduction}
We study smooth nonconvex problems of the form
\begin{equation}
\label{eq:problem}
\min_{x \in \R^d} \left\{f(x) \eqdef \E{f_{\xi}(x)}\right\}\, ,
\end{equation}
 where the randomness comes from the selection of data points and is represented  by the index $\xi$. If the number of indices $n$ is finite, then we talk about {\em empirical risk minimization}  and $\E{f_{\xi}(x)}$ can be written in the finite-sum form, $\nicefrac{1}{n}\sum_{i = 1}^n f_i(x)$. If $n$ is not finite or if it is infeasible to process the entire dataset, we are in the \emph{online learning} setting, where one  obtains independent samples of $\xi$ at each step. We assume that an optimal solution $x^\star$ of \eqref{eq:problem} exists and its value is finite: $f(x^\star)>-\infty$. 
 
\subsection{The many faces of stochastic gradient descent} 

We start with a brief review of relevant aspects of gradient-based optimization algorithms. Since the number of functions $n$ can be large or even infinite, algorithms that process subsamples are essential. The canonical example is Stochastic Gradient Descent (\texttt{SGD})~\citep{nemirovsky1983problem, nemirovski2009robust, SGD-AS}, in which updates are based on single data points or small batches of points.  The terrain around the basic SGD method has been thoroughly explored in recent years, resulting in theoretical and practical enhancements such as  Nesterov acceleration \citep{allen2017katyusha}, Polyak momentum \citep{polyak1964some, sutskever2013importance}, adaptive step sizes \citep{duchi2011adaptive, kingma2014adam, reddi2019convergence, malitsky2019adaptive}, distributed optimization \citep{ma2017distributed, alistarh2017qsgd, stich2018local}, importance sampling \citep{zhao2015stochastic, qu2015quartz}, higher-order optimization \citep{tripuraneni2018stochastic, kovalev2019stochastic}, and several other useful techniques. 

\begin{table*}[t!]
 \centering 
 \scriptsize
\caption{Complexity to reach an $\eps$-approximate first-order stationary points ($\E{\norm{\nabla f(x)}^2} \leq \eps^2$ with $L, \sigma^{2}, \Delta_f = O(1)$).}\label{tab:nonconvex}
\begin{tabular}{|c|c|c|c| }
\hline 
\textbf{Method} & \textbf{Complexity} & \textbf{Required knowledge} \\
\hline
\hline
GD \cite{nesterov2018lectures}  & $\cO\left(\frac{n}{\epsilon^2} \right)$ & $L$ \\
\hline
SVRG \cite{reddi2016stochastic}  & $\cO\left( n + \frac{n^{2/3}}{\epsilon^2} \right)$ & $L$ \\
\hline
SCSG  \cite{lei2017non} & $\td{\cO}\left( \frac{1}{\eps^{10/3}} \wedge \frac{n^{2/3}}{\epsilon^2} \right)$  & $L$ \\
\hline
\multirow{2}{*}{SNVRG \cite{zhou2018stochastic}}  & $\td{\cO}\left(\frac{1}{\epsilon^3} \wedge \frac{\sqrt{n}}{\eps^{2}} \right)$  & $L, \sigma^2, \eps$ \\
& $\td{\cO}\left( n + \frac{\sqrt{n}}{\epsilon^2}   \right)$  & $L$ \\
\hline
SARAH Variants  \cite{fang2018spider}  & \multirow{2}{*}{$\mathcal{O}\left( n + \frac{\sqrt{n}}{\epsilon^2} \right)$}  & \multirow{2}{*}{$L$} \\
 \cite{wang2018spiderboost, nguyen2019finite} & & \\
\hline
{\color{red}\textbf{Q-Geom-SARAH }} (Theorem~\ref{thm:quadratic_complexity}) & $\textcolor{red}{\td{\O}\lb \left\{n^{3/2} + \frac{\sqrt{n}}{\mu}\right\}\wedge \frac{1}{\eps^{3}}\wedge \frac{\sqrt{n}}{\eps^{2}}  \rb}$ & {\color{red} $L$ } \\
\hline
{\color{red}\textbf{E-Geom-SARAH }} (Theorem~\ref{thm:exponential_complexity}) & $\textcolor{red}{\td{\O}\bigg(\lb\frac{1}{\mu\wedge \eps}\rb^{2(1 + \delta)} \wedge \left\{n + \frac{\sqrt{n}}{\mu}\right\} \wedge  \frac{1}{\eps^{4}}\wedge \frac{\sqrt{n}}{\eps^{2}}\bigg)}$ & {\color{red} $L$} \\
\hline
  {\color{red}\textbf{Non-adaptive Geom-SARAH}} (Theorem~\ref{thm:non-adaptive}) & $\textcolor{red}{\O\lb \left\{\frac{1}{\eps^{4/3}(\mu\wedge \eps)^{2/3}}\wedge n\right\} + \frac{1}{\mu}\left\{\frac{1}{\eps^{4/3}(\mu\wedge \eps)^{2/3}}\wedge n\right\}^{1/2}\rb}$ & {\color{red} $L, \sigma^2, \eps, \mu$} \\
\hline
\end{tabular}
\end{table*}

A particularly productive approach to enhancing SGD has been to make use of \textit{variance reduction}, in which the classical stochastic gradient direction is modified in various ways so as to drive the variance of the gradient estimator towards zero. This significantly improves the convergence rate and may also enhance the quality of the output solution. The first variance-reduction method  was \texttt{SAG}~\citep{roux2012stochastic}, closely followed by many more, for instance, \texttt{SDCA}~\citep{shalev2013stochastic},  \texttt{SVRG}~\citep{johnson2013accelerating}, \texttt{S2GD}~\citep{konevcny2013semi},  \texttt{SAGA}~\citep{defazio2014saga}, \texttt{FINITO}~\citep{defazio2014finito},  $\epsilon\mathcal{N} $-\texttt{SAGA}~\citep{hofmann2015variance}, q\texttt{-SAGA}~\citep{hofmann2015variance}, \texttt{QUARTZ}~\citep{qu2015quartz}, \texttt{SCSG}~\citep{lei2017non}, \texttt{SARAH}~\citep{nguyen2017sarah}, \texttt{S2CD}~\citep{S2CD}, \texttt{k-SVRG}~\citep{raj2018k}, \texttt{SNVRG}~\citep{zhou2018stochastic}, \texttt{JacSketch}~\citep{JacSketch}, \texttt{Spider}~\citep{fang2018spider}, \texttt{SpiderBoost}~\citep{wang2018spiderboost}, \texttt{L-SVRG}~\citep{kovalev2019don} and \texttt{GJS}~\citep{GJS}. A unified analysis of many of these methods can be found in \citet{sigma_k}. 

\subsection{The dilemma of parameter tuning}

Formally, each iteration of vanilla and variance-reduced SGD methods can be written in the generic form
\begin{equation}
\label{eq:iter}
x^+ = x - \eta g, 
\end{equation}
where $x\in \R^d$ is the current iterate, $\eta>0$ is a step size and $g\in \R^d$ is a stochastic estimator of the true gradient $\nabla f(x)$.

A major drawback of many such methods is their dependence on parameters that are unlikely to be known in a real-world machine-learning setting. For instance, they may require the knowledge of  a uniform bound on the variance or second moment of the stochastic estimators of the gradient which is simply not available, and might not even hold in practice. Moreover, some algorithms perform well in either low precision or  high precision regimes and in order to make them perform well in all regimes, they require knowledge of extra parameters, such as target accuracy, which may be difficult to tune.  Another related issue is the lack of adaptivity of many SGD variants to different modelling regimes. For example, in order to obtain good theoretical and experimental behavior for non-convex $f$, one needs to run a custom variant of the algorithm if the function is known to satisfy some extra assumptions such as the Polyak-\L{}ojasiewicz (PL) inequality. 
As a consequence, practitioners are often forced to spend valuable time and resources tuning various parameters and hyper-parameters of their methods, which poses serious issues in implementation and practical deployment.

\begin{table*}[t!]
 \centering 
 \scriptsize
\caption{Complexity to reach an $\eps$-approximate solution ( $\E{f(x) - f(x^\star)}\leq \eps^2$ with $L, \sigma^{2}, \Delta_f = O(1)$).}\label{tab:pl}
\begin{tabular}{|c|c|c|c| }
\hline 
\textbf{Method} & \textbf{Complexity} & \textbf{Required Knowledge} \\
\hline
\hline
SVRG \cite{li2018simple}  & $\td{\cO}\left( n + \frac{n^{2/3}}{\mu}\right)$ & $L$ \\
\hline
\multirow{2}{*}{SCSG  \cite{lei2017non}} & $\td{\cO}\left( \left( \frac{1}{\mu\epsilon^2} \wedge n \right) + \frac{1}{\mu} \left( \frac{1}{\mu\epsilon^2} \wedge n \right)^{2/3}\right) $  & $L, \sigma^2, \eps, \mu$ \\
& $\td{\cO}\left( n + \frac{n^{2/3}}{\mu}\right)$  & $L$ \\
\hline
\multirow{2}{*}{SNVRG \cite{zhou2018stochastic}}  & $\td{\cO}  \left(\left( \frac{1}{\mu\epsilon^2} \wedge n \right) + \frac{1}{\mu} \left( \frac{1}{\mu\epsilon^2} \wedge n \right)^{1/2}  \right)$  & $L, \sigma^2, \eps, \mu$ \\
& $\td{\cO} \left( n + \frac{\sqrt{n}}{\mu} \right)$  & $L$ \\
\hline
SARAH Variants  \cite{wang2018spiderboost, nguyen2019finite}  & $\td{\O}\left( n + \frac{1}{\mu^2}\right)$  & $L$ \\
\hline
{\color{red}\textbf{Q-Geom-SARAH }} (Theorem~\ref{thm:quadratic_complexity}) & $\textcolor{red}{\td{\O}\lb \lb\frac{1}{\mu^{2}\wedge \mu\eps^{2}}\rb^{3(1 + \delta)/2}\wedge \left\{n^{3/2} + \frac{\sqrt{n}}{\mu}\right\}\rb}$ & {\color{red} $L$ } \\
\hline
  {\color{red}\textbf{E-Geom-SARAH }} (Theorem~\ref{thm:exponential_complexity}) & $\textcolor{red}{\td{\O}\bigg(\lb\frac{1}{\mu^{2}\wedge \mu\eps^{2}}\rb^{1 + \delta}\wedge \left\{n + \frac{\sqrt{n}}{\mu}\right\}\bigg)}$ & {\color{red} $L$} \\
\hline
\end{tabular}
\end{table*}

\subsection{The search for adaptive methods} 

The above considerations motivate us to impose some algorithm design restrictions so as to  resolve the aforementioned issues. First of all, good algorithms should be {\em adaptive} in the sense that they should perform comparably to methods with tuned parameters without an a-priori knowledge of the optimal parameter settings. In particular, in the non-convex regime, we might wish to design an algorithm
that does not invoke nor need any bound on the variance  of the stochastic gradient, or any predefined target accuracy in its implementation. In addition, we should desire algorithms which perform well if the Polyak-Lojasiewicz PL constant  (or strong convexity parameter) $\mu$ happens to be large and yet are able to converge even if $\mu = 0$; all automatically, without the need for the method to be altered by the practitioner. 

There have been several works on this topic, originating from works studying asymptotic rate for \texttt{SGD} with stepsize $\cO(t^{-\alpha})$  for $\alpha \in (\nicefrac{1}{2},1)$~\cite{ruppert1988efficient, polyak1990new, polyak1992acceleration} up to the most recent paper~\cite{lei2019adaptivity} which focuses on convex optimization \citep[e.g.][]{moulines2011non, bach2013non, flammarion2015averaging, dieuleveut2017harder, xu17, levy2018online, chen2018sadagrad, xu2019accelerate, vaswani2019painless, lan2019unified, hazan2019revisiting}.

This line of research has shown that  algorithms with better complexity can be designed in a finite-sum setting with some levels of adaptivity, generally using the previously mentioned technique--variance reduction. Unfortunately, while these algorithms show some signs of adaptivity, e.g., they do not require the knowledge of $\mu$, they usually fail to adapt to more than one  regimes at once: strongly-convex vs convex loss functions, non-convex vs gradient-dominated regime and low vs high precision. To the best of our knowledge, the only paper that tackles multiple such issues is the work of \citet{lei2019adaptivity}. However, even this work does not provide full adaptivity as it focuses on the convex setting. We are not aware of any work which manages to provide a fully adaptive algorithm in the non-convex setting.

\subsection{Contributions}

In this work we present a new method---the {\em geometrized  stochastic recursive gradient} (\texttt{Geom-SARAH}) algorithm---that exhibits adaptivity to the PL constant, target accuracy and to the variance of stochastic gradients. \texttt{Geom-SARAH} is a double-loop procedure  similar to the \texttt{SVRG} or \texttt{SARAH} algorithms. Crucially, our algorithm does not require the computation of the full gradient in the outer loop as performed by other methods, but makes use of stochastic estimates of gradients in both the outer loop and the inner loop. In addition, by exploiting a randomization technique ``geometrization'' that allows certain terms to telescope across the outer loop and the inner loop, we obtain a significantly simpler analysis. As a byproduct,  this allows us to obtain adaptivity, and our rates either match  the known lower bounds~\cite{fang2018spider} or achieve the same rates as existing state-of-the-art specialized methods, perhaps up to a logarithmic factor; see Table \ref{tab:nonconvex} and \ref{tab:pl} for the comparison of two versions of \texttt{Geom-SARAH} with existing methods. On a side note, we develop a non-adaptive version of \texttt{Geom-SARAH} (the last row of Table \ref{tab:nonconvex}) that strictly outperforms existing methods in PL settings. Interestingly, when $\eps\sim \mu$, our complexity even beats the best available rate for strongly convex functions \citep{allen2018make}.

We would like to point out that our notion of adaptivity is different from the one pursued by  algorithms such as AdaGrad~\cite{duchi2011adaptive} or Adam~\cite{kingma2014adam,reddi2019convergence}, where they focus on the geometry of the loss surface. In our case, we focus on adaptivity to different parameters and regimes.

\section{Preliminaries}


\subsection{Basic notation and definitions}
We use $\norm{\cdot}$ to denote standard Euclidean norm, we write either $\min\{a,b\}$ (resp.\ $\max\{a, b\}$) or $a \wedge b$ (resp.\ $a\vee b$) to denote minimum and maximum, and we use standard big $\cO$ notation to leave out constants\footnote{As implicitly assumed in all other works, we use $\cO(\log x)$ and $\cO(1 / x)$ as abbreviations of $\cO((\log x) \vee 1)$ and $\cO((1 / x)\vee 1)$. For instance, the term $\cO(1/\eps)$ should be interpreted as $\cO((1/\eps)\vee 1)$ and the term $\cO(\log n)$ should be interpreted as $\cO((\log n)\vee 1)$.}.\label{fn:O} We adopt the computational cost model of the IFO framework introduced by \citet{agarwal2014lower} in which upon  query $x$, the IFO oracle samples $i$ and out outputs the pair $(f_i(x), \nabla f_i(x))$. A single such query incurs a unit cost.

\begin{assumption}
\label{as:smooth}
The stochastic gradient of $f$ is $L$-Lipschitz in expectation. That is, 
\begin{equation}
\label{eq:smoothness}
\E{\norm{\nabla f_{\xi}(x) - \nabla f_{\xi}(y)}} \leq L \norm{x-y}, \quad \forall x,y \in \reals^d.
\end{equation}
\end{assumption}

\begin{assumption}
\label{as:var}
The stochastic gradient of $f$ has uniformly bounded variance. That is, there exists $ \sigma^2 >0$ such that
\begin{equation}
\label{eq:var}
\E\norm{\nabla f_{\xi}(x) - \nabla f(x)}^{2} \leq \sigma^2, \quad \forall x\in \reals^d.
\end{equation}
\end{assumption}

\begin{assumption}
\label{as:p-l}
$f$ satisfies the PL condition\footnote{Functions satisfying this condition are sometimes also called \textit{gradient dominated}.} with parameter $\mu\geq 0$. That is, 
\begin{equation}
\label{eq:p-l}
\norm{\nabla f(x)}^2 \geq  2\mu (f(x) - f(x^\star)), \quad \forall  x \in \reals^d,
\end{equation}
where $x^\star = \arg\min f(x)$. 

\end{assumption}

We denote $\Delta_f \eqdef f(x_0) - f(x^\star)$ to be functional distance to optimal solution.


For  non-convex objectives, our goal is to output an $\epsilon$-approximate first-order stationary point, which is summarized in the following definition. 
\begin{definition}\label{def:gradval}
We say that $x\in \reals^d$ is an $\epsilon$-approximate first-order stationary point of \eqref{eq:problem} if 
\begin{equation*}
\norm{\nabla f(x)}^2 \leq \epsilon^2.
\end{equation*}
\end{definition}

For a gradient dominated function, the quantity of the interest is the functional distance from an optimum,  characterized in the following definition.
\begin{definition}\label{def:funcval}
We say that $x\in \reals^d$ is  an $\epsilon$-accurate solution of \eqref{eq:problem} if 
\begin{equation*}
 f(x) - f(x^\star) \leq \epsilon^2.
\end{equation*}
\end{definition}

\subsection{Accuracy independence and almost universality}

We review two fundamental definitions introduced by \citet{lei2019adaptivity} that serve as a building block for  desirable ``parameter-free" optimization algorithms. We refer to the first property as $\eps$-independence.

\begin{definition}
An algorithm is \textbf{$\eps$-independent} if it guarantees convergence at all target accuracies $\eps>0$.
\end{definition}

This is a crucial property as the desired target accuracy is usually not known a-priori. Moreover, an $\eps$-independent algorithm can provide convergence to any precision without the need for a manual adjustment of the algorithm or its parameters. To illustrate this, we consider \texttt{Spider}~\cite{fang2018spider} and \texttt{Spiderboost}~\cite{wang2018spiderboost} algorithms. Both of these enjoy the same complexity $\cO(n + \nicefrac{\sqrt{n}}{\eps^2})$ for non-convex smooth functions, but the stepsize for \texttt{Spider} is $\eps$-dependent, making it  impractical as this value is often hard to tune.

The second property is inspired by the notion of \textit{universality}~\citep{nesterov2015universal}, requiring for an algorithm to not rely on any a-priori knowledge of smoothness or any other parameter such as the bound on variance. 

\begin{definition}
 An algorithm is \textbf{almost universal} if it only requires the knowledge of the smoothness parameter $L$.
\end{definition}

There are several algorithms that satisfy both properties for smooth non-convex optimization, including  \texttt{SAGA}, \texttt{SVRG}~\cite{reddi2016stochastic}, \texttt{Spiderboost}~\cite{wang2018spiderboost}, \texttt{SARAH}~\cite{nguyen2017sarah}, and \texttt{SARAH-SGD}~\cite{tran2019hybrid}. Unfortunately, these algorithms are not able to provide a good result in both low and high precision regimes, and in order to perform well, they require the knowledge of extra parameters. {\em This is not the case for our algorithm which is both almost universal and  $\eps$-independent.} Moreover, our method is adaptive to the PL constant $\mu$, and to low and high precision regimes.

\subsection{Geometric distribution}

Finally, we introduce an important technical tool behind the design of our algorithm, the \textit{geometric distribution}, denoted by $N\sim \mathrm{Geom}(\gamma)$. Recall that 
\begin{equation*}
\Prob(N=k) = \gamma^k (1-\gamma), \forall k = 1,2, \dots, 
\end{equation*}
where an elementary calculation shows that
 \begin{equation*}
\EE{ \mathrm{Geom}(\gamma)}{N} = \frac{\gamma}{1 - \gamma}.
\end{equation*}

We use the geometric distribution mainly due to its following property, which helps us to significantly simplify the analysis of our algorithm.
\begin{lemma}\label{lem:geom}
Let $N\sim \mathrm{Geom}(\gamma)$. Then for any sequence $D_{0}, D_{1}, \ldots$ with $\E{|D_{N}|} < \infty$,
\begin{equation}
\E{ D_{N} - D_{N + 1}} = \frac{1}{\E{N}}(D_{0} - \E{D_{N}}).
\end{equation}
\end{lemma}

\begin{remark}\label{rem:geom}
  The requirement $\E |D_{N}| < \infty$ is essential. A useful sufficient condition is $|D_{k}| = O(\mathrm{Poly}(k))$ because a geometric random variable has finite moments of any order.
\end{remark}

\section{Algorithm}

\begin{algorithm}[!h]
\begin{algorithmic}
  \STATE {\bfseries Input:} stepsizes $\{\eta_j\}$, big-batch sizes $\{B_j\}$, expected inner-loop queries $\{m_j\}$, mini-batch sizes $\{b_j\}$, initializer $\tx_0$, tail-randomized fraction $\delta$
 \FOR{$j=1,\dots \ceil{(1+\delta)T}$}
	\STATE $\jkn{x} = \tx_{j-1}$
	\STATE Sample $J_j$, $|J_j| = B_j$
     \STATE $\jkn{v} = \frac{1}{B_j}\sum_{i \in  J_j} \nabla f_i(\jkn{x})$
     \STATE Sample $N_j \sim \text{Geom}(\gamma_j)$ s.t. $\E{N_j} = \nicefrac{m_j}{b_j}$
       \FOR{$k=0,\dots,N_j - 1$}
       \STATE $\jkp{x} = \jk{x} - \eta_j \jk{v}$
        \STATE Sample $\jk{I}$, $|\jk{I}| = b_j$
        \STATE $\displaystyle \jkp{v} = \frac{1}{b_j}\sum_{i \in  \jk{I}}( \nabla f_i(\jkp{x}) - \nabla f_i(\jk{x}) ) + \jk{v}$
 	\ENDFOR
        \ENDFOR
  \STATE Generate $\cR(T)$ supported on $\{T, \ldots, \ceil{(1 + \delta)T}\}$ with $\displaystyle \Prob(\cR(T) = j) = \nicefrac{\etaj\mj}{\sum_{j=T}^{\lceil (1 + \delta)T\rceil} \etaj\mj}$
  \STATE {\bfseries Output:} $\tx_{\cR(T)}$
\end{algorithmic}  
\caption{Geom-SARAH}
\label{alg:sc_sarah}
\end{algorithm}

The algorithm that we propose can be seen as a combination of the structure of \texttt{SCSG} methods~\cite{lei2016less, lei2017non} and the \texttt{SARAH} biased gradient estimator
\[
\jkp{v} = \frac{1}{\bj}\sum_{i \in  \jk{I}}\lb \nabla f_i(\jkp{x}) - \nabla f_i(\jk{x})\rb + \jk{v}
\]
 due to its recent success in the non-convex setting. Our algorithm consists of several epochs. In each epoch, we start with an initial point $\jkn{x}$ from which the gradient estimator is computed using $B_j$ sampled indices, which is not necessarily the full gradient as in the case of classic \texttt{SARAH} or \texttt{SVRG} algorithm. After this step, we incorporate geometrization of the inner-loop, where the epoch length is sampled from a geometric distribution with predefined mean $m_j$ and in each step of the inner-loop, the \texttt{SARAH} gradient estimator with batch size $b_j$ is used to update the current solution estimate. At the end of each epoch, the last point is taken as the initial estimate for consecutive epoch. The output of our algorithm is then a random iterate $\tx_{\cR(T)}$, where the index $\cR(T)$ is sampled such that  $\Prob(\cR(T) = j) = \nicefrac{\etaj\mj}{\sum_{j=T+1}^{\lceil (1 + \delta)T\rceil} \etaj\mj}$ for $j = T, \dots, \lceil (1 + \delta)T\rceil$. Note that $\cR(T) = T$ when $\delta = 0$.  
 This procedure can be seen tail-randomized iterate which as an analogue of tail-averaging in the convex-case~\cite{rakhlin2011making}. For functions $f$ with finite support (finite $n$), the sampling procedure in Algorithm~\ref{alg:sc_sarah} is sampling without replacement. For the infinite support, this is just $B_j$ or $b_j$ i.i.d. samples, respectively. The pseudo-code is shown in Algorithm~\ref{alg:sc_sarah}.

Define $T_{g}(\eps)$ and $T_{f}(\eps)$ as the iteration complexity to find an $\eps$-approximate first-order stationary point and an $\eps$-approximate solution, respectively:
  \begin{equation}
    \label{eq:Tgeps}
    T_{g}(\eps) \eqdef \min\{T \;:\; \E \norm{\nabla f(\tx_{\cR(T)})}^{2}\le \eps^{2},  \forall T' \ge T\}.
  \end{equation}
  and
  \begin{equation}
    \label{eq:Tfeps}
    T_{f}(\eps) \eqdef \min\{T \;:\; \E (f(\tx_{\cR(T)}) - f(x^\star))\le \eps^{2}, \forall T' \ge T\},
  \end{equation}
  where $\tx_{\cR(T)}$ is output of given algorithm.
  
The query complexity to find an $\eps$-approximate first-order stationary point and an $\eps$-approximate solution are defined as $\comp_{g}(\eps)$ and $\comp_{f}(\eps)$, respectively.

\begin{remark}
Note that in Definition~\ref{def:funcval} and equation~\eqref{eq:Tfeps}, we use $\eps^2$ instead of commonly used $\eps$. We decided to use $\eps^2$ because we examine both  $\eps$-approximate first-order stationary point and an $\eps$-approximate solution together and these two are connected via \eqref{eq:p-l}, which justifies our choice to use $\eps^2$ for both. This implies that for fair comparison with previous methods, one needs to either use $\sqrt{\eps}$ instead of $\eps$ for our rates or $\eps^2$ instead of $\eps$ for previous works.
\end{remark}

 It is easy to see that
\begin{align*}
\E \comp_{g}(\eps) &= \sum_{j=1}^{\lceil (1 + \delta)T_{g}(\eps)\rceil}(2\mj + \Bj) \\
\E \comp_{f}(\eps) &= \sum_{j=1}^{\lceil (1 + \delta)T_{f}(\eps)\rceil}(2\mj + \Bj), 
\end{align*}

\section{Convergence Analysis}

We conduct the analysis of our method in the way, where we first look at the progress of inner cycle for which we establish bounds on the norm of the gradient, which is subsequently used to prove convergence of the full algorithm. We assume $f$ to be $L$-smooth and satisfy PL condition with $\mu$ which can be equal to zero.

\subsection{One Epoch Analysis}

We start from a one-epoch analysis that connects consecutive iterates. It lays the foundation for complexity analysis. The analysis is similar to \citet{elibol2020variance} and presented in Appendix \ref{app:proofs}.
\begin{theorem}\label{thm:one_epoch}
  Assume that $\eta_j$ and $\mj$ are picked such that
  \begin{equation}
    \label{eq:eta_cond}
    2\etaj L \le \min\left\{1, \frac{\bj}{\sqrt{\mj}}\right\}.
  \end{equation}
  Then under assumptions \ref{as:smooth} and \ref{as:var}, 
  \[\E \|\nabla f(\tx_{j})\|^{2}\le \frac{2\bj}{\etaj \mj}\E(f(\tx_{j-1}) - f(\tx_{j})) + \frac{\sigma^{2}I(\Bj < n)}{\Bj}\]
\end{theorem}

\subsection{Complexity Analysis}
\label{sec:complexity}

We consider two versions of our algorithm--\texttt{Q-Geom-SARAH} and \texttt{E-Geom-SARAH}. These two version differs only in the way how we select the big batch size $B_j$ for our algorithm. For \texttt{Q-Geom-SARAH}, we select quadratic growth of $B_j$ and \texttt{E-Geom-SARAH}, this is selected to be exponential. The convergence guarantees follow with all proofs relegated to Appendix \ref{app:proofs}.

\begin{theorem}[\texttt{Q-Geom-SARAH}]
\label{thm:quadratic_complexity}
  Set the hyperparameters as
  \[\etaj = \frac{\bj}{2L\sqrt{\mj}}, \quad \bj\le \sqrt{\mj}, \quad \mj = \Bj = j^2\wedge n, \quad \delta = 1.\]
  Let
\[\Delta = \Delta_{f} + \frac{\sigma^{2}}{L}\log n.\]  
  Then
  \begin{align*}
    &\E\comp_{g}(\eps) = \O\bigg(\left\{\frac{L^{3}}{\mu^{3}} + \frac{\sigma^{3}}{\eps^{3}} + \log^{3}\lb\frac{\mu \Delta}{\eps^{2}}\rb\right\} \wedge\\
    & \left\{n^{3/2} + \lb n + \frac{\sqrt{n}L}{\mu}\rb\log\lb\frac{L\Delta}{\sqrt{n}\eps^{2}}\rb\right\}\wedge \frac{(L\Delta)^{3/2}}{\eps^{3}}\wedge \frac{\sqrt{n}L\Delta}{\eps^{2}}\bigg),
  \end{align*}
  and
\begin{align*}  
  \E\comp_{f}(\eps) =& \O\bigg( \left\{\frac{L^{3}}{\mu^{3}} + \frac{\sigma^{3}}{\mu^{3/2}\eps^{3}} + \log^{3}\lb\frac{\Delta}{\eps^{2}}\rb\right\}   \\
  &\wedge\left\{n^{3/2} + \lb n + \frac{\sqrt{n}L}{\mu}\rb\log\lb\frac{\Delta}{\eps^{2}}\rb\right\}\bigg),
\end{align*}
\end{theorem}
where $\O$ only hides universal constants.
\begin{remark}
  Theorem \ref{thm:quadratic_complexity} continues to hold if $\eta L = \theta \bj / \sqrt{\mj}$ for any $0 < \theta < 1/2$ and $\mj, \Bj \in [a_{1}j^{2}, a_{2}j^{2}]$ for some $0 < a_{1} < a_{2} < \infty$ for sufficiently large $j$.
\end{remark}
First we notice that the logarithm factors are smaller than $\log (\Delta / \eps)$ due to the multiplier $\mu$ and $1 / \sqrt{n}$. If $\mu$ is small or $n$ is large, they can be as small as $O(1)$. In general, to ease comparison, we ignore the logarithm factors. Then Theorem \ref{thm:quadratic_complexity} implies that $\Delta  = \td{\O}\lb\Delta_{f} + \sigma^{2}\rb$,
\begin{equation}
  \label{eq:compg}
  \begin{split}
  \E\comp_{g}(\eps) =& \td{\O}\bigg(\left\{\frac{L^{3}}{\mu^{3}} + \frac{\sigma^{3}}{\eps^{3}} \right\}\wedge \left\{n^{3/2} + \frac{\sqrt{n}L}{\mu}\right\} \\
  & \wedge \frac{(L\Delta_{f})^{3/2} + \sigma^{3}}{\eps^{3}}\wedge \frac{\sqrt{n}(L\Delta_{f} + \sigma^{2})}{\eps^{2}}
\bigg),
 \end{split}
\end{equation}
and
\begin{equation}
  \label{eq:compf}
  \E\comp_{f}(\eps) = \td{\O}\lb \left\{\frac{L^{3}}{\mu^{3}} + \frac{\sigma^{3}}{\mu^{3/2}\eps^{3}}\right\} \wedge \left\{n^{3/2} + \frac{\sqrt{n}L}{\mu}\right\}\rb.
\end{equation}
Theorem \ref{thm:quadratic_complexity} shows an unusually strong adaptivity in that the last two terms match the state-of-the-art complexity \citep[e.g.][]{fang2018spider} for general smooth non-convex optimization while it may be further improved when PL constant is large without any tweaks. 
  
There is a gap between the complexity of \texttt{Q-Geom-SARAH} and the best achievable rate by non-adaptive algorithms in the PL case. This motivates us to consider another variant of \texttt{Geom-SARAH} that performs better for PL objectives while still have guarantees for general smooth nonconvex objectives. Let $\es$ denote the exponential square-root, i.e.
\begin{equation}
  \label{eq:es}
  \es(x) = \exp\{\sqrt{x}\}.
\end{equation}
It is easy to see that
\[\log x = \O(\es(\log x))\]
 and 
\[ \es(\log x) = \O\lb x^{a}\rb \mbox{ for any }a > 0.\]

\begin{theorem}[\texttt{E-Geom-SARAH}]
\label{thm:exponential_complexity}
    Fix any $\alpha > 1$ and $\delta \in (0, 1]$. Set the hyperparameters as
    \[\etaj = \frac{\bj}{2L\sqrt{\mj}},\quad  \bj\le \sqrt{\mj},\]
    where
    \[  \mj = \alpha^{2j}\wedge n,\quad \Bj = \lceil\alpha^{2j}\wedge n\rceil.\]
Let
    \[\Delta' = \Delta_{f} + \frac{\sigma^{2}}{L}.\]
    Then
  \begin{align*}
    &\E\comp_{g}(\eps) =\\
    & \O\bigg(\left\{\frac{L^{2(1 + \delta)}}{\mu^{2(1 + \delta)}}\es\lb 2\log_{\alpha}\left\{\frac{\mu\Delta'}{\eps^{2}}\right\}\rb + \lb\frac{\sigma}{\eps}\rb^{2(1 + \delta)}\right\}\log^{2}n\\
&\wedge \left\{n\log \lb\frac{L}{\mu}\rb + \lb n + \frac{\sqrt{n}L}{\mu}\rb\log\lb\frac{L\Delta'}{\sqrt{n}\eps^{2}}\rb\right\} \\
&\wedge  \frac{(L\Delta')^{2}}{\delta^{2}\eps^{4}}\wedge \frac{\sqrt{n}L\Delta'\log n}{\delta \eps^{2}}\bigg),
  \end{align*}
  and
  \begin{align*}
    &\E\comp_{f}(\eps) = \\
     &\O\bigg(\left\{\frac{L^{2(1 + \delta)}}{\mu^{2(1 + \delta)}}\es\lb 2\log_{\alpha}\left\{\frac{\Delta'}{\eps^{2}}\right\}\rb + \lb\frac{\sigma^{2}}{\mu\eps^{2}}\rb^{1 + \delta}\right\}\log^{2}n\\
    & \wedge \left\{n\log \lb\frac{L}{\mu}\rb + \lb n + \frac{\sqrt{n}L}{\mu}\rb\log\lb\frac{\Delta'}{\eps^{2}}\rb\right\}\bigg),
  \end{align*}
  where $\O$ only hides universal constants and constants that depend on $\alpha$.
\end{theorem}

Ignoring the logarithm factors and $\es(\cdot)$ factors and setting $\delta = O(1)$, Theorem \ref{thm:exponential_complexity} implies
\begin{align*}
&\E\comp_{g}(\eps) = \td{\O}\bigg(\left\{\frac{L^{2(1 + \delta)}}{\mu^{2(1 + \delta)}} + \lb\frac{\sigma}{\eps}\rb^{2(1 + \delta)}\right\} \\
& \quad \wedge \left\{n + \frac{\sqrt{n}L}{\mu}\right\}\wedge  \frac{(L\Delta_f)^{2} + \sigma^{4}}{\eps^{4}}\wedge \frac{\sqrt{n}(L\Delta_f + \sigma^{2})}{\eps^{2}}\bigg),
\end{align*}

and
\begin{align*}
&\E\comp_{f}(\eps) = \\
&\quad\td{\O}\lb\left\{\frac{L^{2(1 + \delta)}}{\mu^{2(1 + \delta)}} + \lb\frac{\sigma^{2}}{\mu\eps^{2}}\rb^{1 + \delta}\right\} \wedge \left\{n + \frac{\sqrt{n}L}{\mu}\right\}\rb.
\end{align*}

Note that in order to provide convergence result for all three cases we need $\delta$ to be arbitrarily small positive constant, thus one might almost ignore factor $1+\delta$ in the complexity results. Recall that $\delta = 0$ implies $\cR(T) = T$ meaning that the output of an algorithm is the last iterate, which is common setting, e.g. for \texttt{Spiderboost} or \texttt{SARAH}, under assumption $\mu > 0$. 

\begin{figure*}[t!]
\centering
\hfill
\subfigure[\texttt{mushrooms}]
{\includegraphics[width=0.325\textwidth]{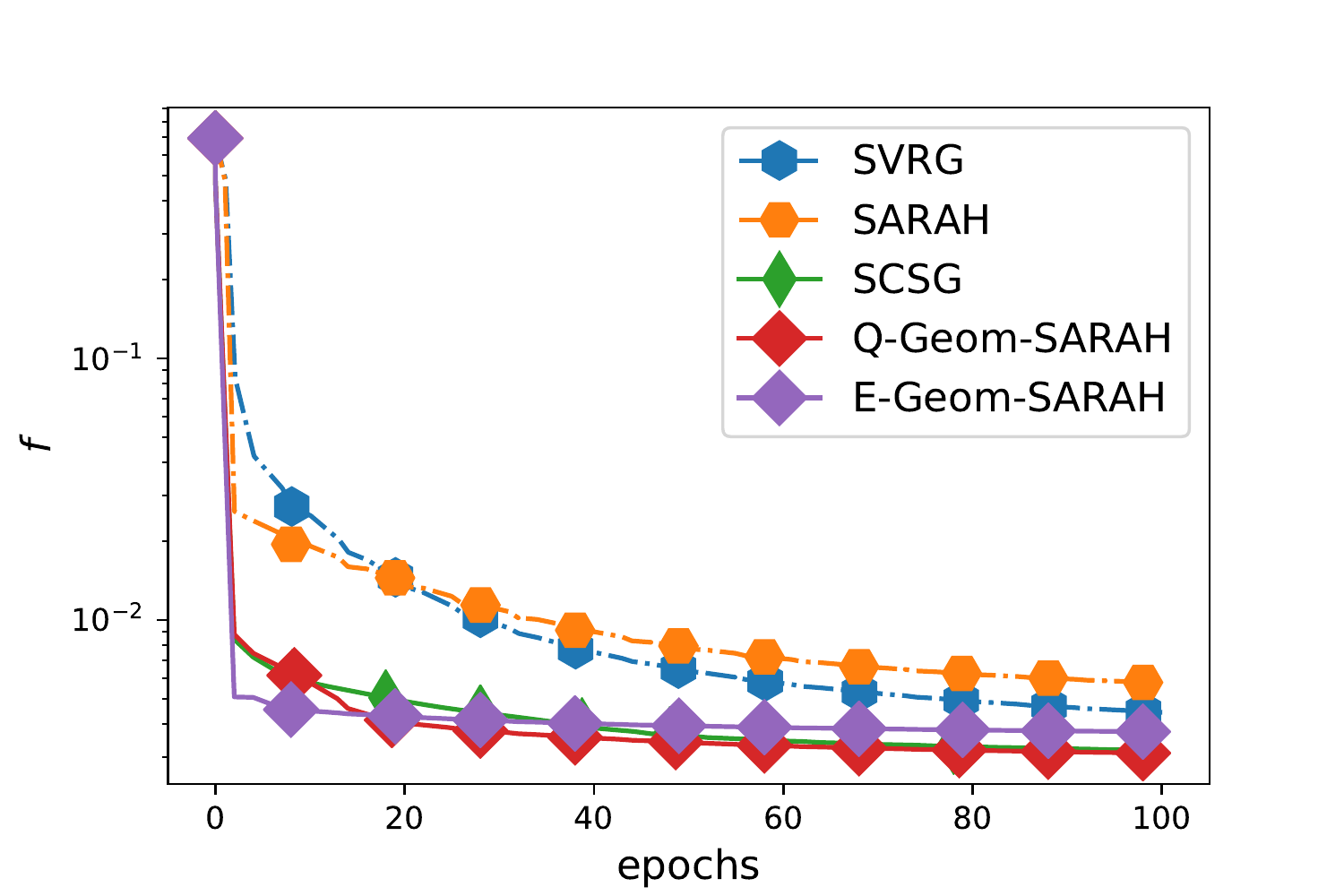}}
\hfill
\subfigure[\texttt{w8a}]
{\includegraphics[clip,width=0.325\textwidth]{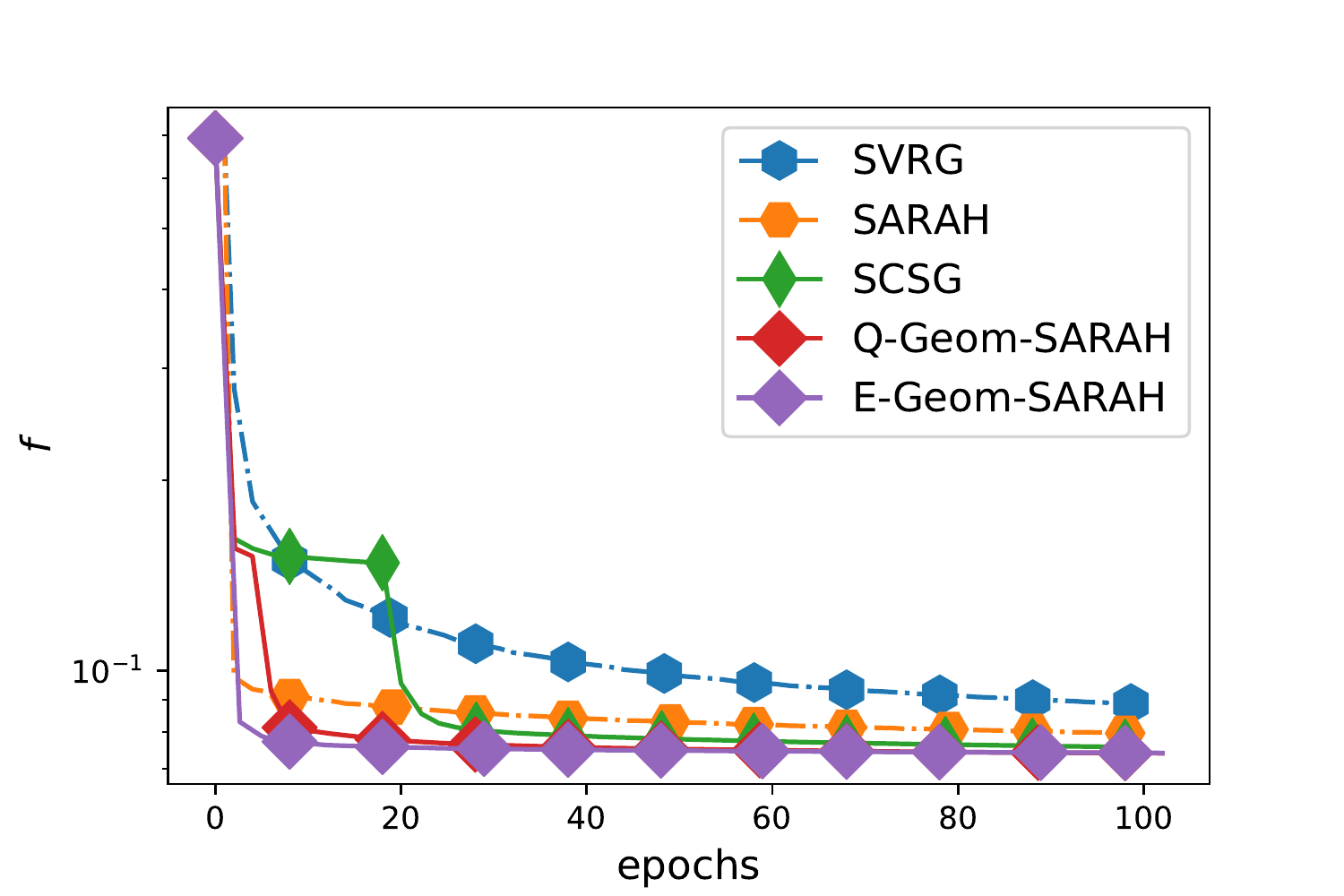}}
\hfill
\subfigure[\texttt{ijcnn1}]
{\includegraphics[width=0.325\textwidth]{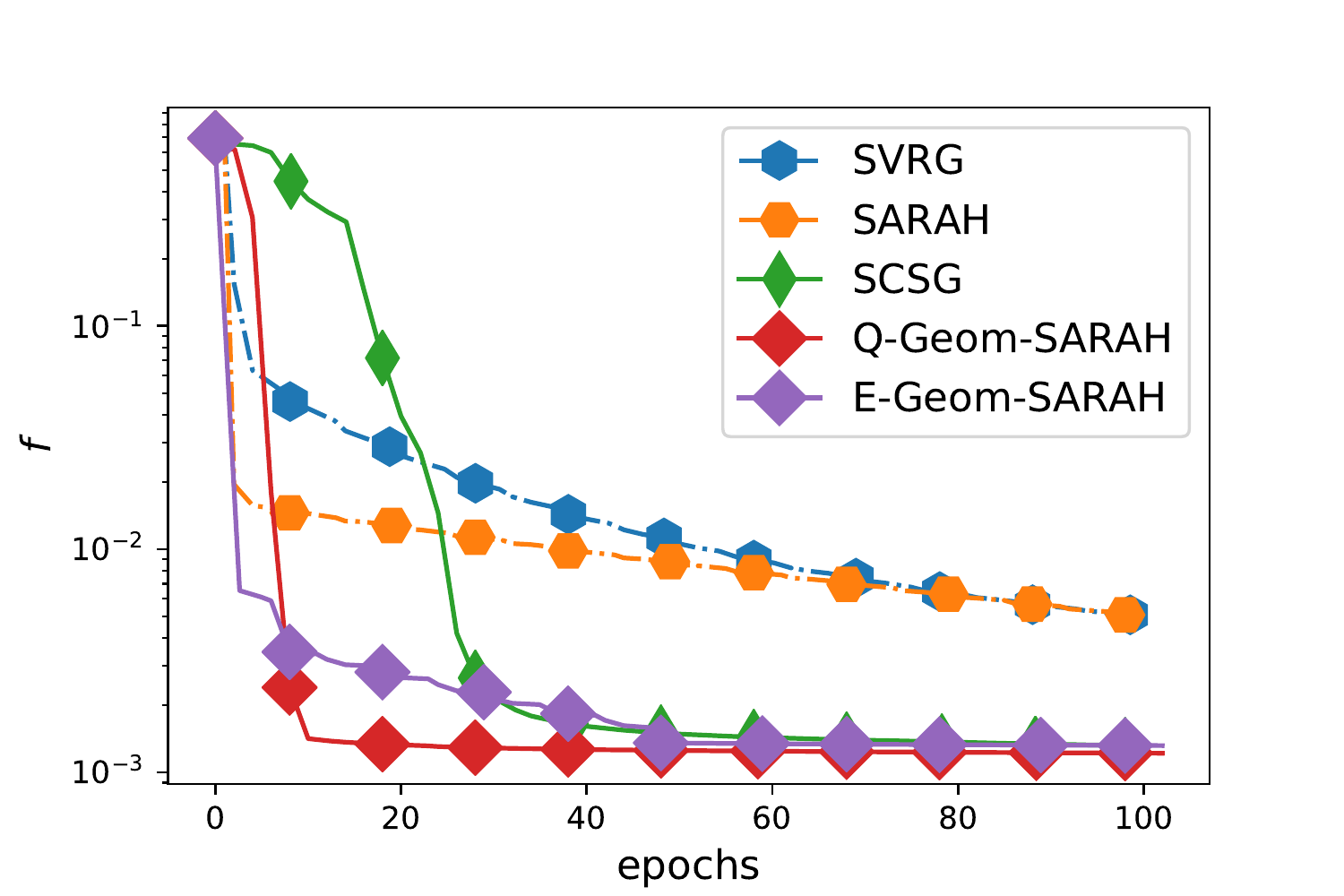}}
\hfill\null
\caption{Comparison of convergence with respect to functional value for different high precision VR methods.\label{fig:fval_high}}
\end{figure*}

\begin{figure*}[t!]
\centering
\hfill
\subfigure[\texttt{mushrooms}]
{\includegraphics[width=0.325\textwidth]{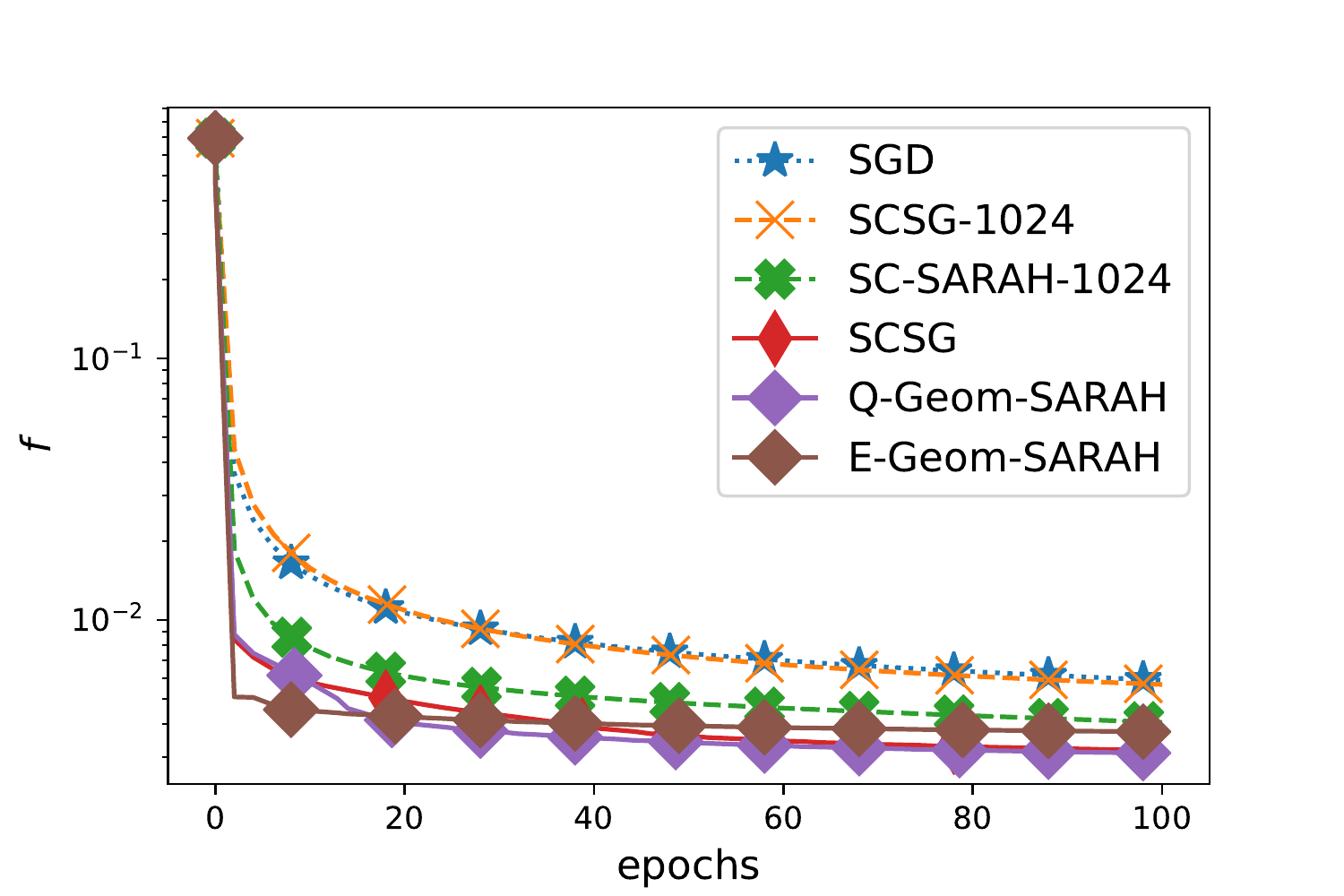}}
\hfill
\subfigure[\texttt{w8a}]
{\includegraphics[width=0.325\textwidth]{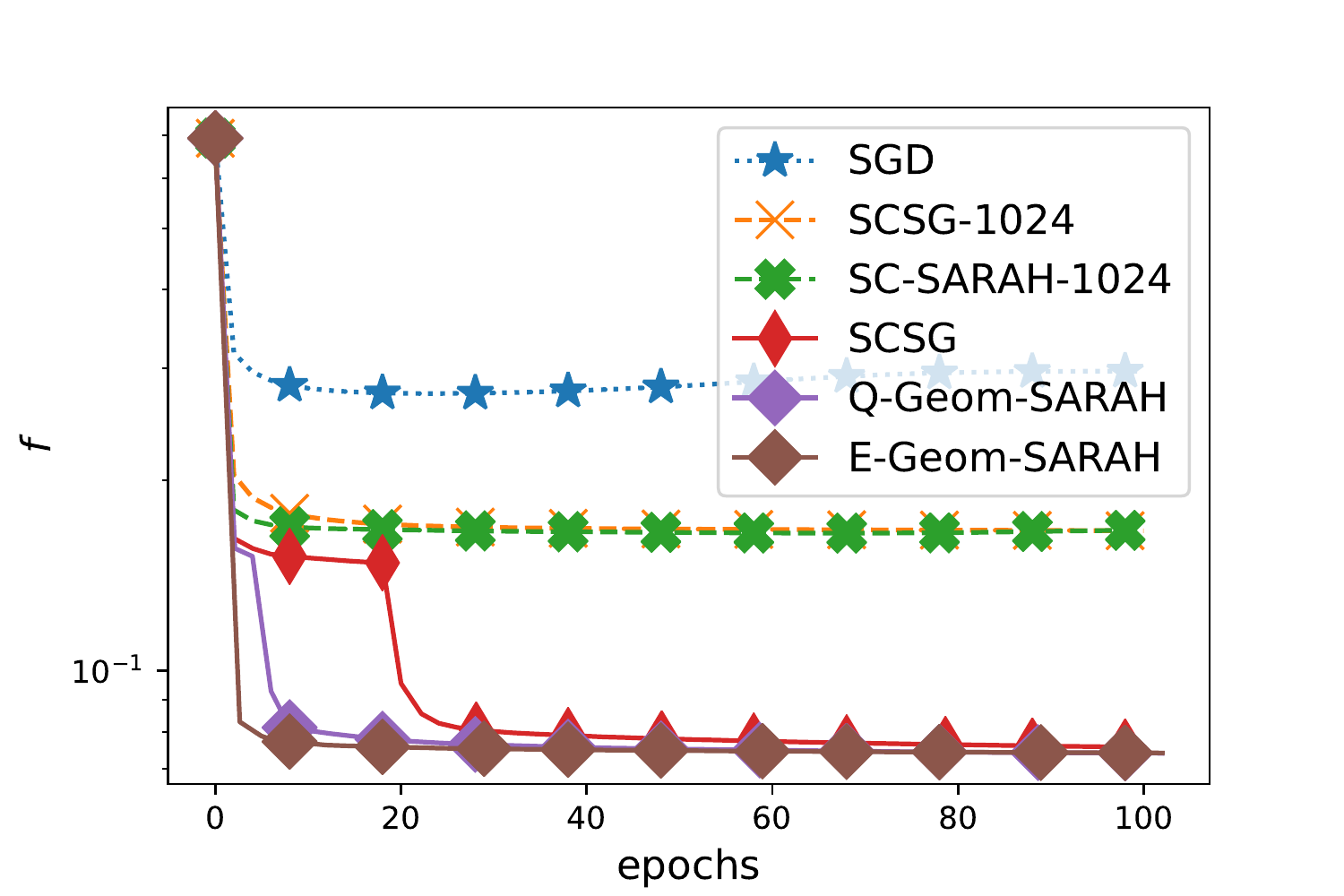}}
\hfill
\subfigure[\texttt{ijcnn1}]
{\includegraphics[width=0.325\textwidth]{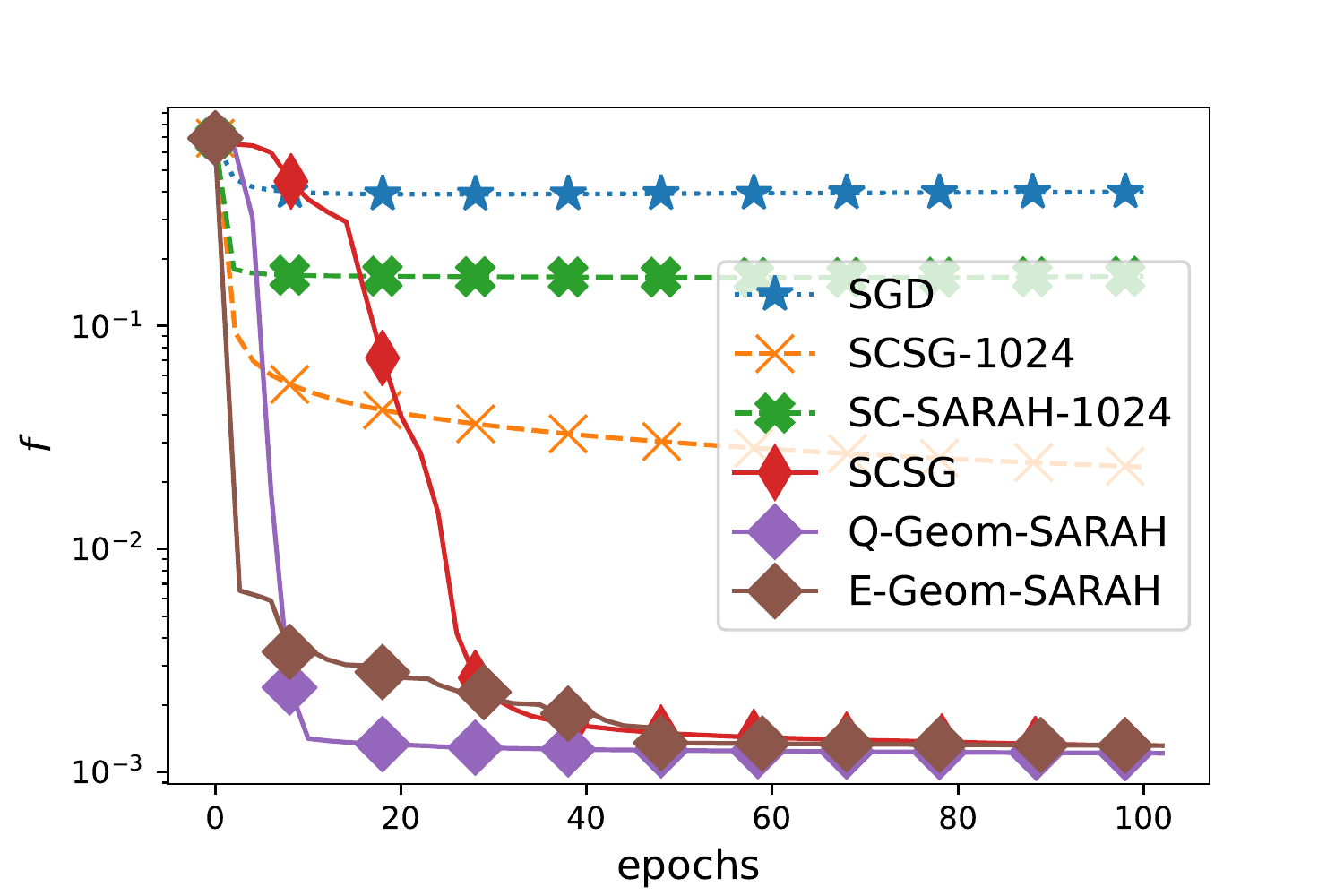}}
\hfill\null
\caption{Comparison of convergence with respect to functional value for different low precision VR methods.\label{fig:fval_low}}
\end{figure*}

\subsection{Better rates for non-adaptive \texttt{Geom-SARAH}}

In this section, we provide the versions of our algorithms, which are neither almost universal nor $\eps$-independent, but they either reach the known lower bounds or best achievable results known in literature. We include this result for two reasons. Firstly, we want to show there is a small gap between results in Section~\ref{sec:complexity} and the best results, which might be obtained. We conjecture that this gap is inevitable. Secondly, our complexity result for the functional gap beats the best known complexity result known in literature which is $\mathcal{O}\left(\log^3 B \left[B + \frac{\sqrt{B}}{\mu} \right] \log(\frac{1}{\epsilon}) \right)$, where $B=\O\lb \frac{\sigma^2}{\mu\epsilon^2} \wedge n\rb$ \citep{zhou2018stochastic}
, where our complexity result does not involve $ \log^3 B$ factor. 
Finally, we obtain very interesting result for the norm of the gradient, which we discuss later in this section. The proofs are relegated into Appendix \ref{app:proofs}.

\begin{theorem}[Non-adaptive]
\label{thm:non-adaptive}
Set the hyperparameters as
\[
\eta_j = \frac{b_j}{2L\sqrt{m_j}}, \quad b_j \leq \sqrt{m_j},\quad B_j = m_j = B.
\]
\begin{enumerate}
\item If $B=\left( \frac{\sigma^2}{4\mu\epsilon^2} \wedge n \right)$ and $\delta = 0$ then
\begin{align*}
&\E\comp_{f}(\eps) = \O\lb \lb B + \frac{\sqrt{B}L}{\mu}\rb\log\lb\frac{\Delta_f}{\eps^{2}}\rb\rb
\end{align*}
\item If $B = \left( \left\{\frac{8\sigma^{2}}{\eps^{2}} + \frac{8\sigma^{4/3}L^{2/3}}{\eps^{4/3}\mu^{2/3}}\right\} \wedge n \right)$ 
  and $\delta = 0$ then
\[
\E\comp_{g}(\eps) =  \O\lb   \lb  B +  \frac{\sqrt{B}L}{\mu}\rb\log\frac{L\Delta_f}{\sqrt{B}\eps^2} \rb,
\]
\end{enumerate}
\end{theorem}

Looking into these result, there is one important thing to note. While these methods reach state-of-the-art performance for PL objectives, they provide no guarantees for the case $\mu = 0$.

For the ease of presentation we assume $\sigma^2, \Delta_f, L = \cO(1)$. For \texttt{Q-Geom-SARAH}, we can see that in term of $\td{\cO}$ notation, we match the best reachable rate in case $\mu = 0$. For the case $\mu > 0$, we see slight degradation in performance for both high and low precision regimes. For \texttt{E-Geom-SARAH}, we can see a bit different results. There is a $\nicefrac{1}{\eps}$ degradation comparing to the low precision case and exact match for high precision case with $\mu = 0$. For the case $\mu>0$, \texttt{E-Geom-SARAH} matches the best achievable rate for high precision and also for in low precision regime in the case when rate is dominated by factor $\nicefrac{1}{\eps^2}$.  Comparison to other methods together with the dependence on parameters can be found in Tables~\ref{tab:nonconvex} and \ref{tab:pl}.

One interesting fact to note is that in the second case of Theorem \ref{thm:non-adaptive}, if $\mu\sim \eps$ and $L, \Delta_f, \sigma^{2} = \O(1)$, $B \sim 1 / \eps^{2}$ and
\[\E \comp_{g}(\eps)= \O\lb \frac{1}{\eps^{2}}\log\lb\frac{1}{\eps}\rb\rb.\]
This is even logarithmically better than the rate $\O(\eps^{-2}\log^{3}(1 / \mu))$ obtained by \citet{allen2018make} for strongly-convex functions. Note that a strongly convex function with modulus $\mu$ is always $\mu$-PL. We plan to further investigate this strong result in the future.

\begin{figure*}[t!]
\centering
\hfill
\subfigure[\texttt{mushrooms}]
{\includegraphics[width=0.32\textwidth]{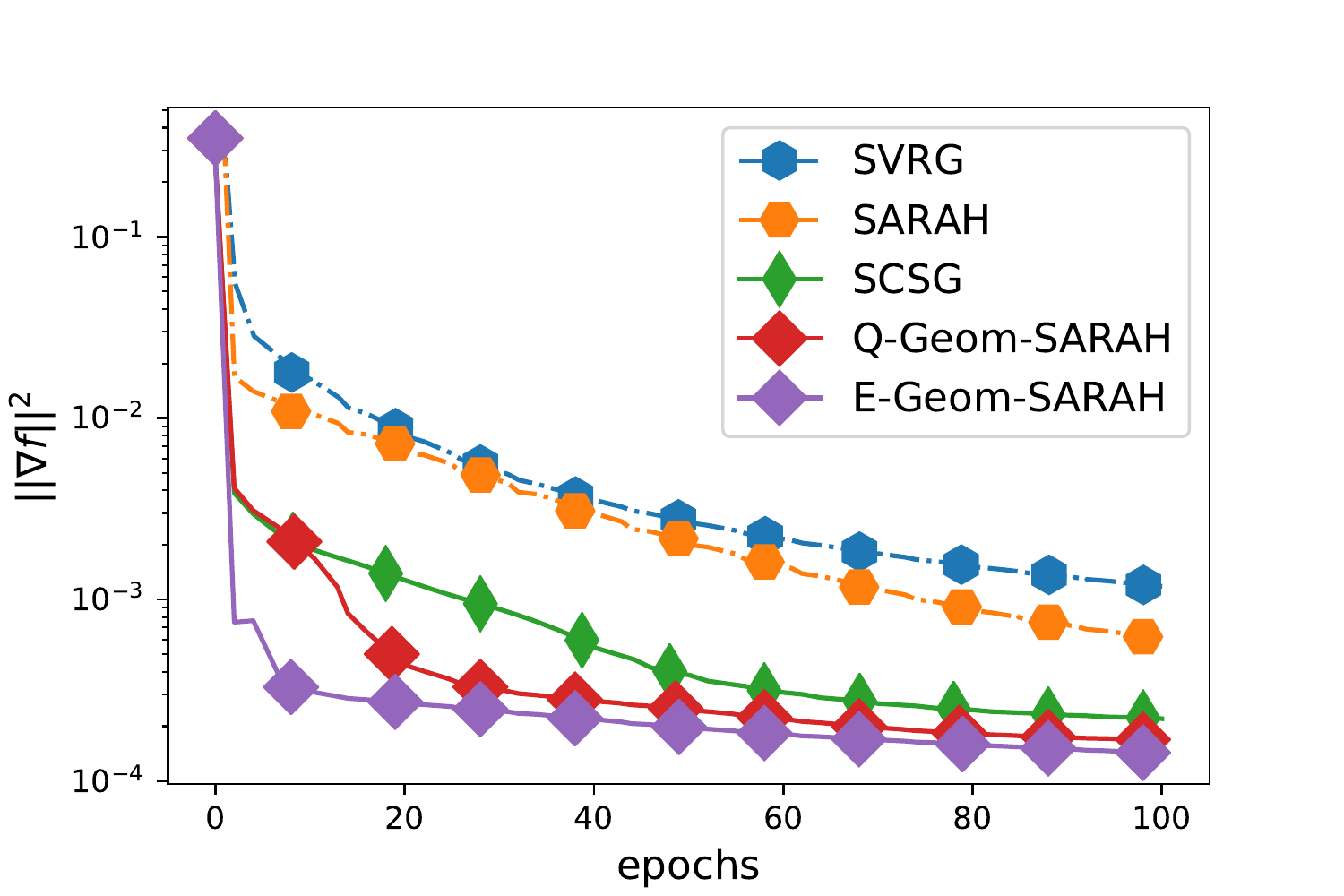}}
\hfill
\subfigure[\texttt{w8a}]
{\includegraphics[width=0.32\textwidth]{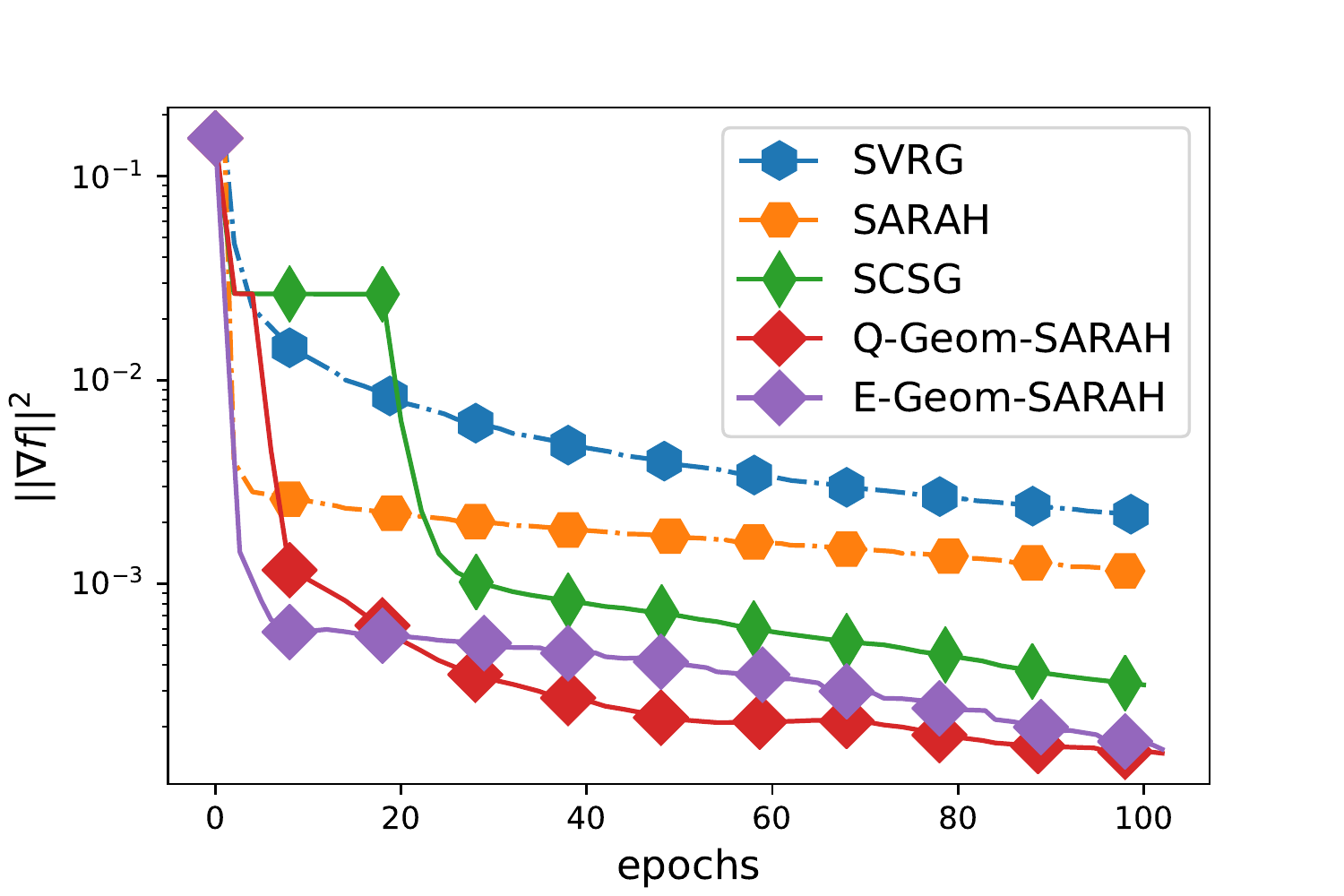}}
\hfill
\subfigure[\texttt{ijcnn1}]
{\includegraphics[width=0.32\textwidth]{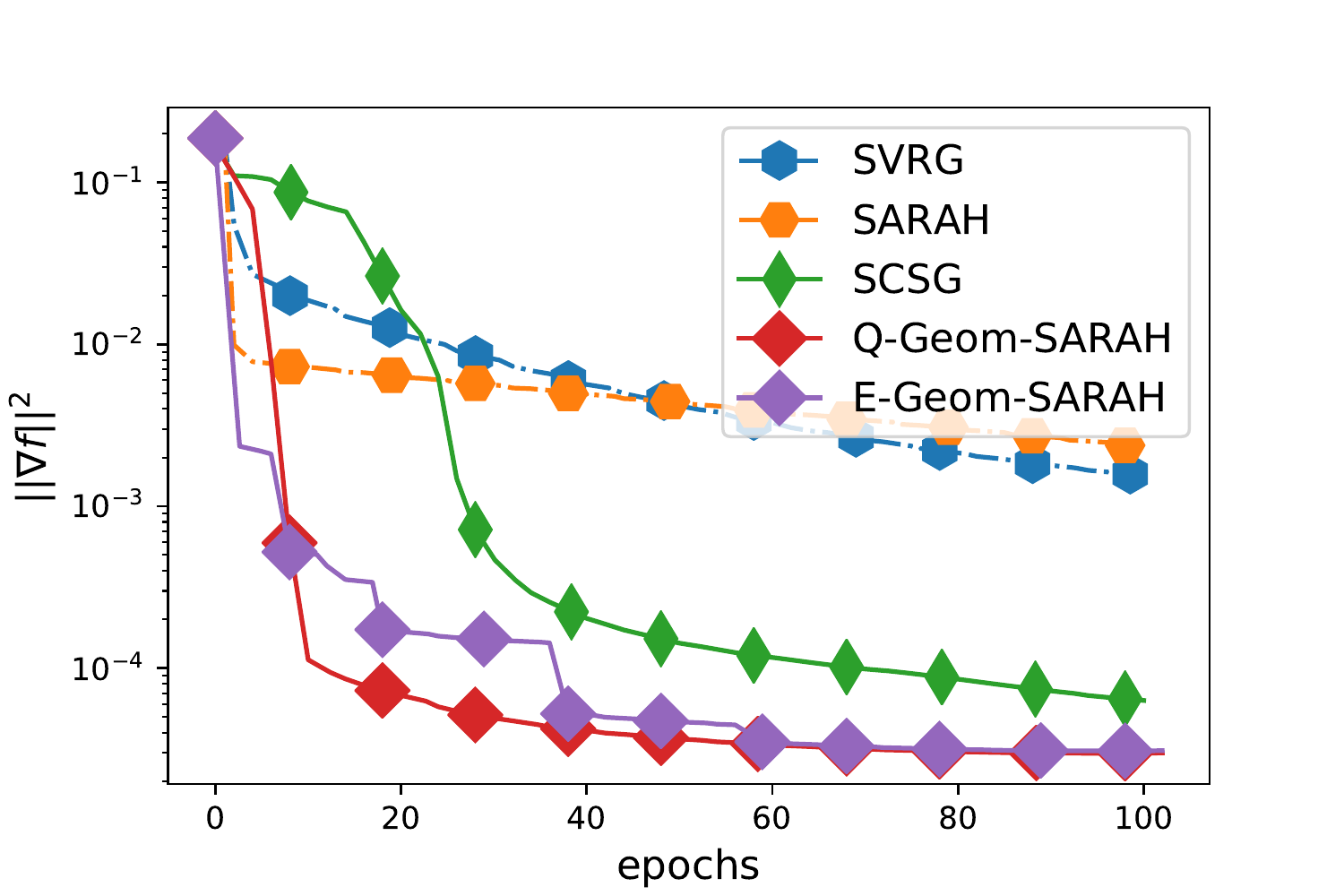}}
\hfill\null
\caption{Comparison of convergence with respect to norm of the gradient  for different high precision VR methods.\label{fig:gval_high}}
\end{figure*}

\begin{figure*}[t!]
\centering
\hfill
\subfigure[\texttt{mushrooms}]
{\includegraphics[width=0.32\textwidth]{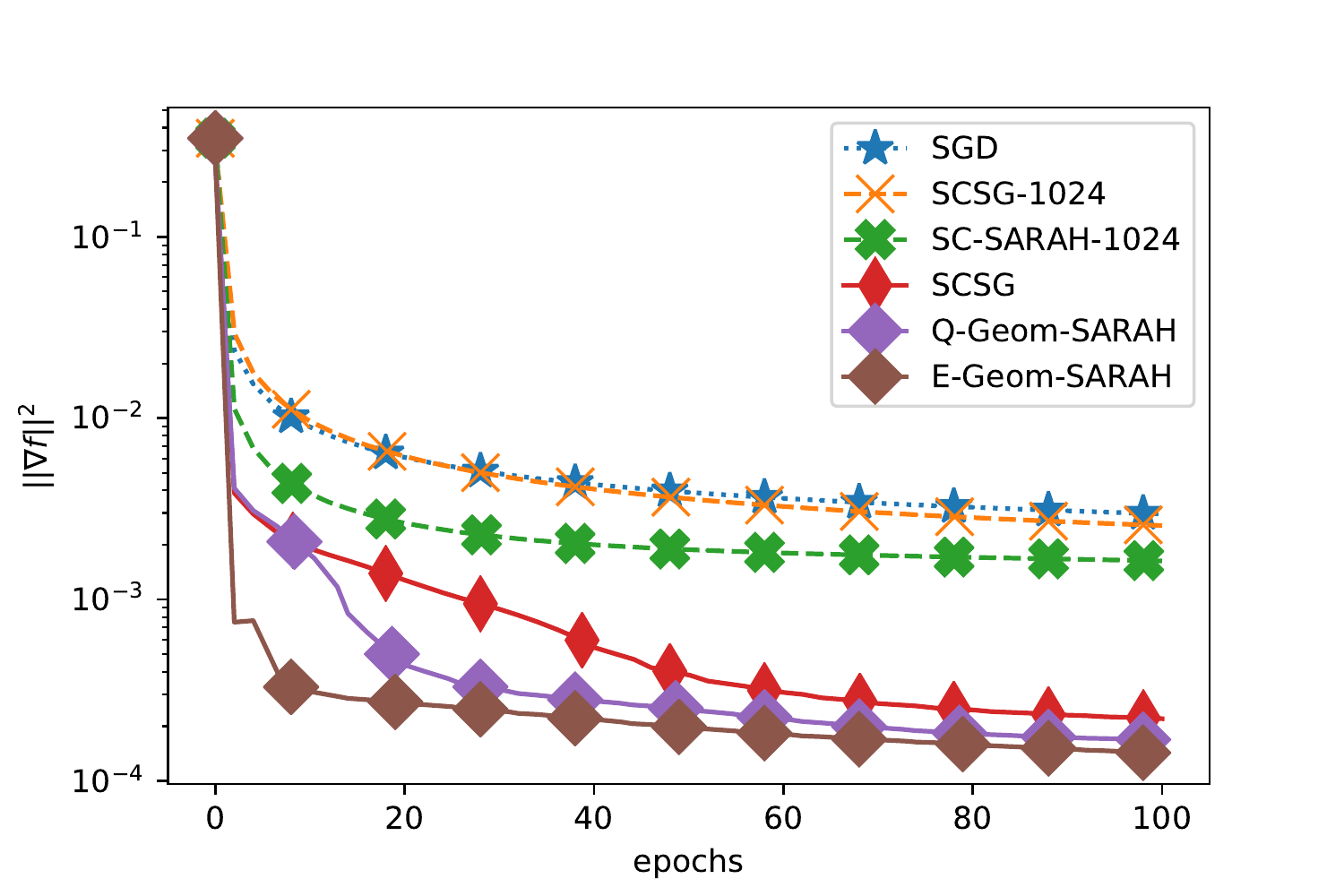}}
\hfill
\subfigure[\texttt{w8a}]
{\includegraphics[width=0.32\textwidth]{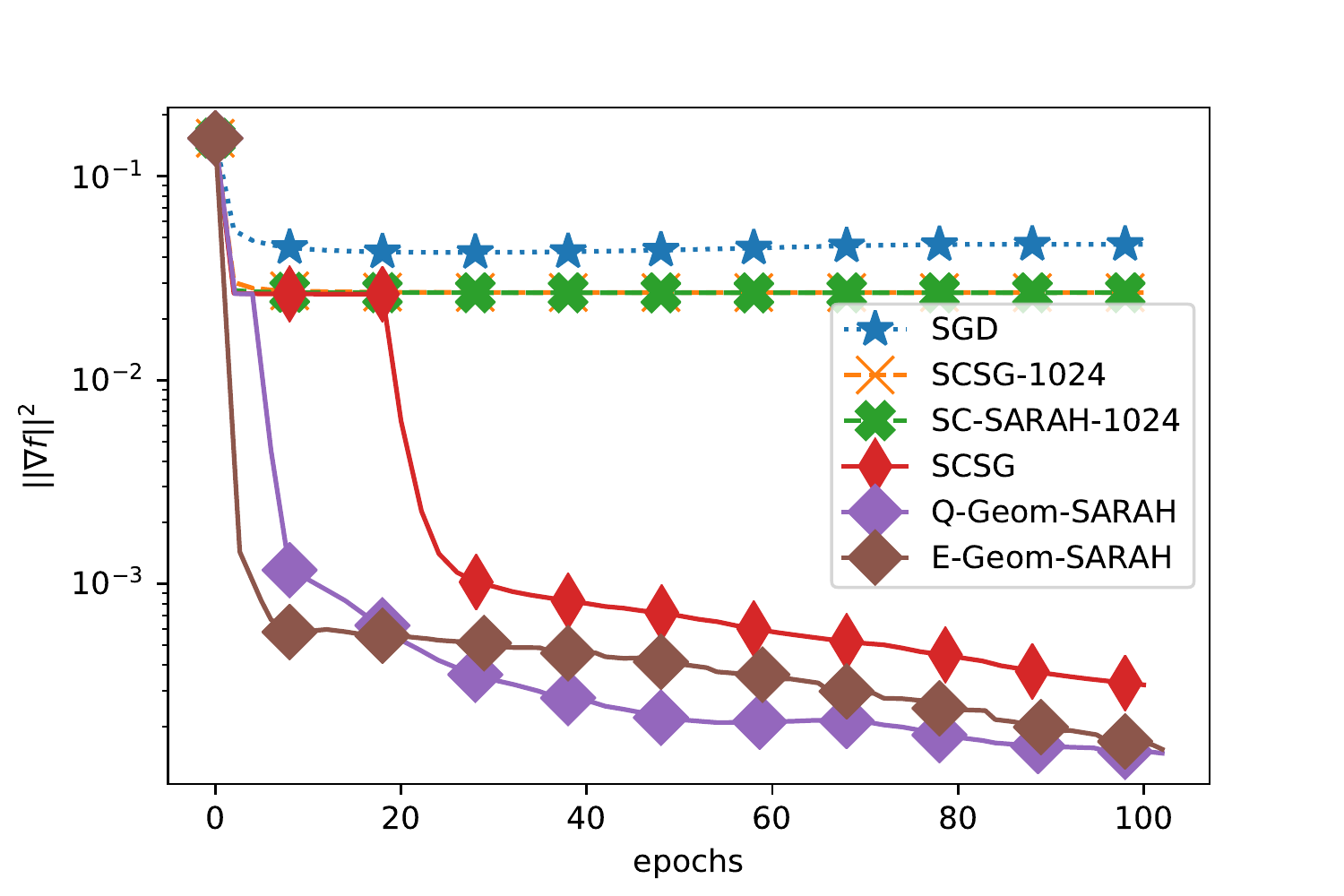}}
\hfill
\subfigure[\texttt{ijcnn1}]
{\includegraphics[width=0.32\textwidth]{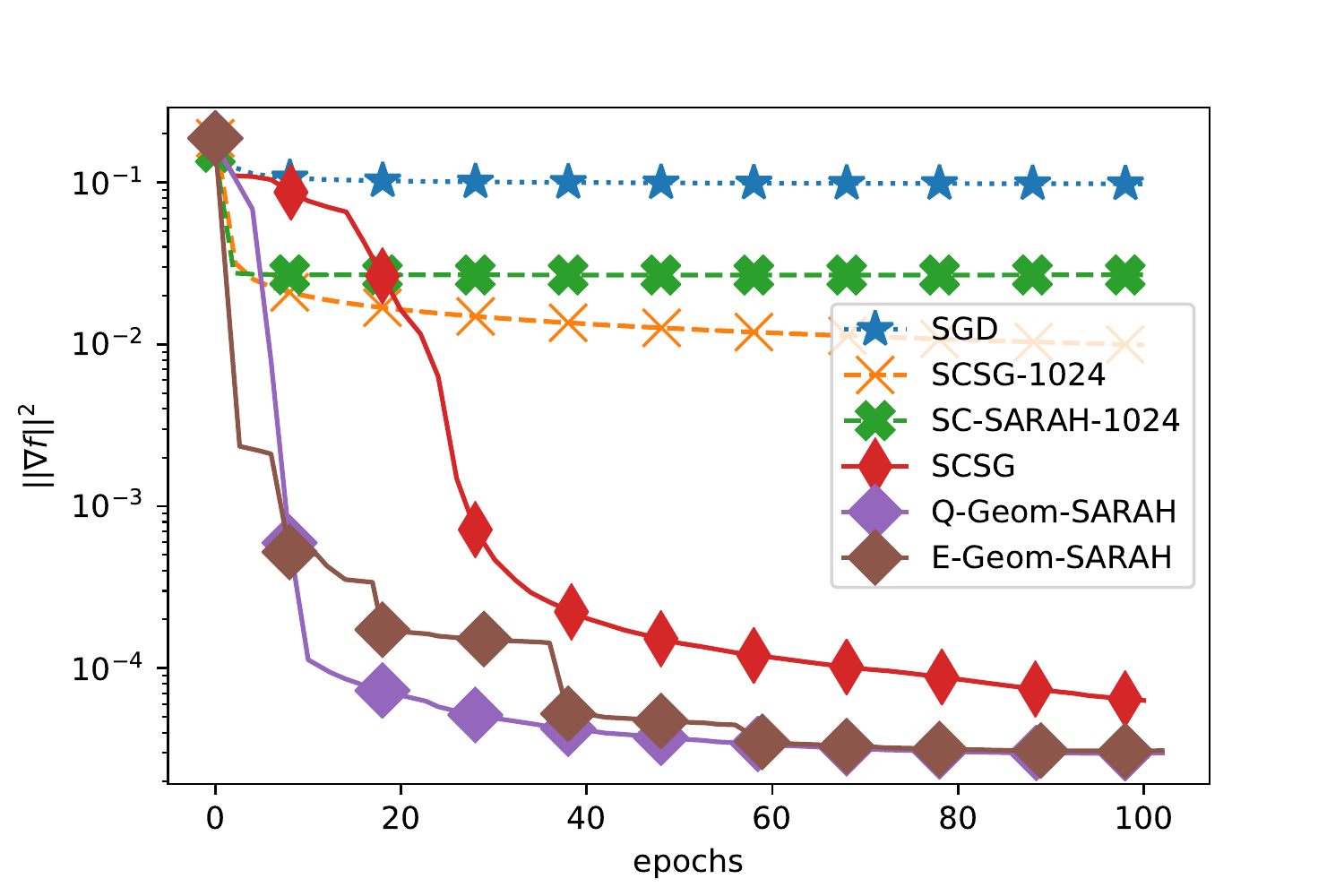}}
\hfill\null
\caption{Comparison of convergence with respect to norm of the gradient for different low precision VR methods.\label{fig:gval_low}}
\end{figure*}

\section{Experiments}

To support our theoretical result, we conclude several experiments using logistic regression with non-convex penalty. The objective that we minimize is of the form
\[
 -\frac{1}{n}\sum_{i=1}^n \left[ f_{i}(x) \eqdef \log\lb1 +e^{-y_iw_i^\top x}\rb + \frac{\lambda}{2} \sum_{j=1}^d \frac{x_j^2}{1+x_j^2}\right],
\]
where $w_i$'s are the features, $y_i$'s the labels and $\lambda > 0$ is a regularization parameter. This fits to our framework with $L_{f_i} = \nicefrac{\norm{a_i}^2}{4} + \lambda$. We compare our adaptive methods against state-of-the-art methods in this framework--\texttt{SARAH}~\cite{nguyen2019finite}, \texttt{SVRG}~\cite{reddi2016stochastic}, \texttt{Spiderboost}~\cite{wang2018spiderboost},  adaptive and fixed version of  \texttt{SCSG}~\cite{lei2017non} with big batch sizes $B = cj^{3/2} \wedge n$ for some constant $c$. We use all the methods with their theoretical parameters. We use \texttt{SARAH} and \texttt{Spiderboost} with constant step size $\nicefrac{1}{2L}$, which implies batch size to be $b = \sqrt{n}$. In this scenario, \texttt{Spiderboost} and \texttt{SARAH} are the same algorithm and we refer to both as \texttt{SARAH}. The same step size is also used for \texttt{SVRG} which requires batch size $b = n^{2/3}$. The same applies to \texttt{SCSG} and our methods and we adjust parameter accordingly, e.g. this applies that for our methods we set $b_j = \sqrt{m_j}$. For \texttt{E-Geom-SARAH}, we chose $\alpha = 2$. We also include \texttt{SGD} methods with the same step size for comparison. All the experiments are run with $\lambda = 0.1$. We use three dataset  from LibSVM\footnote{available on \url{https://www.csie.ntu.edu.tw/~cjlin/libsvmtools/datasets/}}:  \textit{mushrooms} ($n = 8, 124, p = 112$), \textit{w8a} ($n = 49, 749, p = 300$), and \textit{ijcnn1} ($n = 49, 990, p = 22$).

We run two sets of experiments-- low and high precision. Firstly, we compare our adaptive methods with the ones that can guarantee convergence to arbitrary precision $\eps$ -- \texttt{SARAH}, \texttt{SVRG} and adaptive \texttt{SCSG}. Secondly, we conclude the experiment where we compare our adaptive methods against ones that should provide better convergence in low precision regimes-- \texttt{SARAH} and  \texttt{SVRG} with big batch size $B = 1024$, adaptive \texttt{SCSG} and  \texttt{SGD} with batch size equal to $32$. For all the experiments, we display functional value and norm of the gradient with respect to number of epochs (IFO calls divided by $n$). For all Figures~\ref{fig:fval_high}, \ref{fig:fval_low}, \ref{fig:gval_high} and \ref{fig:gval_low}, we can see that our adaptive method perfoms the best in all the regimes and the only method that reaches comparable performance is \texttt{SCSG}.

\section{Conclusion}

We have presented two new methods \texttt{Q-Geom-SARAH} and \texttt{E-Geom-SARAH}, a gradient-based algorithm for the non-convex finite-sum/online optimization problem. We have shown that our methods are both $\eps$-independent and almost-universal algorithms. We obtain these properties via \textit{geometrization} and careful batch size construction. Our methods provide strictly better results comparing to other methods as these are the only methods which can adapt to multiple regimes, i.e. low/high precision or PL with $\mu =0/\mu > 0$. Moreover, we show that the obtained complexity is closed to or even matches the best achievable one in all the regimes. 

\bibliography{literature}
\bibliographystyle{icml2020}

\clearpage

\appendix
\onecolumn
\part*{Appendix}
\setlength{\footskip}{20pt}

\section{Proofs}\label{app:proofs}

\subsection{\textbf{Proof of Lemma \ref{lem:geom}}}
By definition, 
  \begin{align*}
   E (D_{N} - D_{N + 1})  &= \sum_{n\geq 0}(D_{k} - D_{k+1})\cdot \gamma^{k}(1 - \gamma)\\ 
&= (1 - \gamma)( D_{0} - \sum_{k \geq 1}D_{k}(\gamma^{k-1} - \gamma^{k})) = (1 - \gamma)\lb\frac{1}{\gamma}D_{0} - \sum_{k \geq 0}D_{k}(\gamma^{k-1} - \gamma^{k})\rb\\ 
&=  (1 - \gamma)\lb \frac{1}{\gamma}D_{0} - \frac{1}{\gamma}\sum_{k\ge 0}D_{k}\gamma^{k}(1 - \gamma)\rb = \frac{1 -\gamma}{\gamma}  (D_{0} - \E{ D_{N}}),
\end{align*}
where the last equality is implied by the condition that $\E{ |D_{N}|} < \infty$.

In order to use Lemma \ref{lem:geom}, one needs to show $\E{ |D_{N}|} < \infty$. We start with the following lemma as the basis to apply geometrization. The  proof is distracting and relegated to the end of this section. 
\begin{lemma}\label{lem:exp_exist}
Assume that $\etaj L \le 1$. Then $\E |D_{\Nj}^{(s)}| < \infty$ for $s = 1, 2, 3$, where
  \[D_{k}^{(1)} = \E_{j}\norm{\nuj_{k} - \nabla f(\xj_{k})}^{2}, \quad D_{k}^{(2)} = \E_{j}f(\xj_{k}), \quad D_{k}^{(3)} = \E_{j}\norm{\nabla f(\xj_{k})}^{2},\]
  and $\E_{j}$ denotes the expectation over the randomness in $j$-th outer loop. 
\end{lemma}

Based on Lemma \ref{lem:exp_exist}, we prove two lemmas, which helps us to establish the sequence that is used to prove convergence. Throughout the rest of the section we assume that assumption \ref{as:smooth} and \ref{as:var} hold.

\begin{lemma}\label{lem:nu-nablaf}
For any $j$,
  \[\E_{j}\norm{\nuj_{\Nj} - \nabla f(\tx_{j})}^{2}\le \frac{\mj \etaj^{2}L^{2}}{\bj^{2}}\E_{j}\norm{\nuj_{\Nj}}^{2} + \frac{\sigma^{2}I(\Bj\le n)}{\Bj},\]
where $\E_{j}$ denotes the expectation over the randomness in $j$-th outer loop. 
\end{lemma}
\begin{proof}
Let $\E_{j, k}$ and $\Var_{j, k}$ denote the expectation and  variance operator over the randomness of $\sI_{k}$. Since $\sI_{k}$ is independent of $\xj_{k}$,
\[\E_{j, k}\nuj_{k+1} = \nuj_{k} + (\nabla f(\xj_{k+1}) - \nabla f(\xj_{k})).\]
Thus, 
\[\nuj_{k + 1} - \nabla f(\xj_{k + 1}) = \nuj_{k} - \nabla f(\xj_{k}) + \lb \nuj_{k+1} - \nuj_{k} - \E_{j, k} (\nuj_{k+1} - \nuj_{k})\rb.\]  
Since $\sI_{k}$ is independent of $(\nuj_{k}, \xj_{k})$, 
\[\Cov_{j, k}\lb\nuj_{k} - \nabla f(\xj_{k}), \nuj_{k+1} - \nuj_{k}\rb = 0.\]
As a result,
\begin{equation}
  \label{eq:nuj_recursion}
  \E_{j, k} \norm{\nuj_{k + 1} - \nabla f(\xj_{k + 1})}^{2} = \norm{\nuj_{k} - \nabla f(\xj_{k})}^{2} + \Var_{j, k}(\nuj_{k+1} - \nuj_{k}).
\end{equation}
By Lemma \ref{lem:var_sampling},
\begin{align}
  &\Var_{j, k}(\nuj_{k+1} - \nuj_{k}) = \Var\lb\frac{1}{\bj}\sum_{i\in \sI_{k}}(\nabla f_{i}(\xj_{k+1}) - \nabla f_{i}(\xj_{k}))\rb\nonumber\\
  & \le \frac{1}{\bj}\frac{1}{n}\sum_{i=1}^{n}\norm{\nabla f_{i}(\xj_{k+1}) - \nabla f_{i}(\xj_{k}) - (\nabla f(\xj_{k+1}) - \nabla f(\xj_{k}))}^{2}\label{eq:varjk}\\
  & \le \frac{1}{\bj}\frac{1}{n}\sum_{i=1}^{n}\norm{\nabla f_{i}(\xj_{k+1}) - \nabla f_{i}(\xj_{k})}^{2}.\nonumber
\end{align}
Finally by assumption \ref{as:smooth}, 
\[\frac{1}{n}\sum_{i=1}^{n}\norm{\nabla f_{i}(\xj_{k+1}) - \nabla f_{i}(\xj_{k})}^{2}\le L^{2}\norm{\xj_{k+1} - \xj_{k}}^{2} = \etaj^{2}L^{2}\norm{\nuj_{k}}^{2}.\]
By \eqref{eq:nuj_recursion},
\[\E_{j, k} \norm{\nuj_{k + 1} - \nabla f(\xj_{k + 1})}^{2} = \norm{\nuj_{k} - \nabla f(\xj_{k})}^{2} + \frac{\etaj^{2}L^{2}}{\bj}\norm{\nuj_{k}}^{2}.\]
Let $k = \Nj$ and take expectation over all randomness in $\E_{j}$. By Lemma \ref{lem:exp_exist}, we can apply Lemma \ref{lem:geom} on $D_{k} = \E_{j}\norm{\nuj_{k} - \nabla f(\xj_{k})}^{2}$. Then we have
\begin{align}
  0&\le \E_{j}\lb \|\nuj_{\Nj} - \nabla f(\xj_{\Nj})\|^{2} - \|\nuj_{\Nj + 1} - \nabla f(\xj_{\Nj + 1})\|^{2}\rb + \frac{\etaj^{2}L^{2}}{\bj}\E_{j}\|\nuj_{\Nj}\|^{2}\nonumber\\
& = \frac{\bj}{\mj}\E_{j}\lb \|\nuj_{0} - \nabla f(\xj_{0})\|^{2} - \|\nuj_{\Nj} - \nabla f(\xj_{\Nj})\|^{2}\rb + \frac{\etaj^{2}L^{2}}{\bj}\E_{j}\|\nuj_{\Nj}\|^{2}\nonumber.
\end{align}
Finally, by Lemma \ref{lem:var_sampling}, 
\begin{align}
  &\E_{j} \|\nuj_{0} - \nabla f(\xj_{0})\|^{2} \le \frac{\sigma^{2}I(\Bj < n)}{\Bj}.\label{eq:nu-nablaf3}
\end{align}
The proof is then completed.
\end{proof}

\begin{lemma}\label{lem:fx}
For any $j$, 
\[\E_{j} \norm{\nabla f(\tx_{j})}^{2}\le \frac{2\bj}{\etaj \mj}\E_{j}(f(\tx_{j-1}) - f(\tx_{j})) + \E_{j}\norm{\nuj_{\Nj} - \nabla f(\tx_{j})}^{2} - (1 - \etaj L)\E_{j}\norm{\nuj_{\Nj}}^{2},\]
where $\E_{j}$ denotes the expectation over the randomness in $j$-th outer loop.
\end{lemma}
\begin{proof}
  By assumption \eqref{as:smooth}, 
  \begin{align}
    f(\xj_{k+1})
    &\le f(\xj_{k}) + \la \nabla f(\xj_{k}), \xj_{k + 1} - \xj_{k}\ra + \frac{L}{2} \norm{\xj_{k} - \xj_{k+1}}^{2}\nonumber\\
    & = f(\xj_{k}) - \eta\la \nabla f(\xj_{k}), \nuj_{k}\ra + \frac{\etaj^{2}L}{2} \norm{\nuj_{k}}^{2} \nonumber\\
    & = f(\xj_{k}) + \frac{\etaj}{2}\norm{\nuj_{k} - \nabla f(\xj_{k})}^{2} - \frac{\etaj}{2}\norm{\nabla f(\xj_{k})}^{2} - \frac{\etaj}{2}\norm{\nuj_{k}}^{2} + \frac{\etaj^{2}L}{2} \norm{\nuj_{k}}^{2}.\label{eq:fx_recursion}
  \end{align}
  Let $j = \Nj$ and take expectation over all randomness in $\E_{j}$. By Lemma \ref{lem:exp_exist}, we can apply Lemma \ref{lem:geom} with $D_{k} = \E_{j}f(\xj_{k})$ and $D_{k} = \E_{j}\norm{\nabla f(\xj_{k})}^{2}$. Thus,
  \begin{align*}   
    0&\le \E_{j}\lb f(\xj_{\Nj}) - f(\xj_{\Nj+1})\rb + \frac{\etaj}{2}\E_{j}\norm{\nuj_{\Nj} - \nabla f(\xj_{\Nj})}^{2} - \frac{\etaj}{2}\E_{j}\norm{\nabla f(\xj_{\Nj})}^{2} - \frac{\etaj}{2}(1 - \etaj L)\E_{j}\norm{\nuj_{\Nj}}^{2}\\
     & = \frac{\bj}{\mj}\E_{j}\lb f(\xj_{0}) - f(\xj_{\Nj})\rb + \frac{\etaj}{2}\E_{j}\norm{\nuj_{\Nj} - \nabla f(\xj_{\Nj})}^{2} - \frac{\etaj}{2}\E_{j}\norm{\nabla f(\xj_{\Nj})}^{2} - \frac{\etaj}{2}(1 - \etaj L)\E_{j}\norm{\nuj_{\Nj}}^{2}\\   
     & = \frac{\bj}{\mj}\E_{j}\lb f(\tx_{j-1}) - f(\tx_{j})\rb + \frac{\etaj}{2}\E_{j}\norm{\nuj_{\Nj} - \nabla f(\xj_{\Nj})}^{2} - \frac{\etaj}{2}\E_{j}\norm{\nabla f(\tx_{j})}^{2} - \frac{\etaj}{2}(1 - \etaj L)\E_{j}\norm{\nuj_{\Nj}}^{2}.
  \end{align*}
The proof is then completed.
\end{proof}

Theorem \ref{thm:one_epoch} is then proved by combining Lemma \ref{lem:nu-nablaf} and Lemma \ref{lem:fx}.

\begin{proof}[\textbf{Proof of Theorem \ref{thm:one_epoch}}]
  By Lemma \ref{lem:nu-nablaf} and Lemma \ref{lem:fx},
  \begin{align*}
  \E \|\nabla f(\tx_{j})\|^{2}&\le \frac{2\bj}{\etaj \mj}\E(f(\tx_{j-1}) - f(\tx_{j})) + \frac{\sigma^{2}I(\Bj < n)}{\Bj}  - \lb 1 - \etaj L - \frac{\mj \etaj^{2} L^{2}}{\bj^{2}}\rb\E\|\nuj_{\Nj}\|^{2}.
  \end{align*}
  Under condition \eqref{eq:eta_cond},
  \[1-\eta_j L - \frac{\mj(\eta_j L)^2 }{\bj^{2}} \geq 1 - \frac{1}{2} - \frac{1}{4}\ge 0,\]
  which concludes the proof.
\end{proof}

\begin{proof}[\textbf{Proof of Lemma \ref{lem:exp_exist}}]
  By \eqref{eq:varjk} and assumption \ref{as:var}, 
  \begin{align*}
    \Var_{j,k}(\nuj_{k+1} - \nuj_{k})&\le \frac{2}{\bj n}\lb\sum_{i=1}^{n}\norm{\nabla f_{i}(\xj_{k+1}) - \nabla f(\xj_{k+1})}^{2} + \sum_{i=1}^{n}\norm{\nabla f_{i}(\xj_{k}) - \nabla f(\xj_{k})}^{2}\rb\\
    & \le \frac{4\sigma^{2}}{\bj}\le 4\sigma^{2}.
  \end{align*}
  By \eqref{eq:nuj_recursion} and taking expectation over all randomness in epoch $j$, 
  \[\E_{j}\norm{\nuj_{k + 1} - \nabla f(\xj_{k + 1})}^{2}\le \E_{j}\norm{\nuj_{k} - \nabla f(\xj_{k})}^{2} + 4\sigma^{2}.\]
  Then
  \begin{equation}
    \label{eq:Dk1}
    \E_{j}\norm{\nuj_{k} - \nabla f(\xj_{k})}^{2}\le \norm{\nuj_{0} - \nabla f(\xj_{0})}^{2} + 4k\sigma^{2}\le (4k + 1)\sigma^{2} = \mathrm{poly}(k),
  \end{equation}
  where the last inequality uses \eqref{eq:nu-nablaf3}. By remark \ref{rem:geom}, we obtain that $\E |D_{\Nj}^{(1)}| < \infty$.

  ~\\
  \noindent On the other hand, by \eqref{eq:fx_recursion}, since $\etaj L\le 1$,
  \[f(\xj_{k+1}) + \frac{\etaj}{2}\norm{\nabla f(\xj_{k})}^{2} \le f(\xj_{k}) + \frac{\etaj}{2}\norm{\nuj_{k} - \nabla f(\xj_{k})}^{2}\le f(\xj_{k}) + (2k+1)\etaj \sigma^{2},\]
  where the last inequality uses \eqref{eq:Dk1}. Let
  \[\Mj_{k} = f(\xj_{k+1}) - f(x^\star)+ \frac{\etaj}{2}\norm{\nabla f(\xj_{k})}^{2}.\]
  Then
  \[\Mj_{k}\le \Mj_{k-1} + (2k+1)\etaj \sigma^{2}.\]
  Applying the above inequality recursively, we have
  \[\Mj_{k}\le \Mj_{0} + (k^2 + 2k)\etaj \sigma^{2} = \mathrm{poly}(k).\]
  As a result,
  \[0\le f(\xj_{k+1}) - f(x^\star)\le \Mj_{k-1} = \mathrm{poly}(k), \quad 0\le \norm{\nabla f(\xj_{k})}^{2}\le \frac{1}{\etaj} \Mj_{k} = \mathrm{poly}(k).\]
  By remark \ref{rem:geom}, we obtain that
  \[\E |f(\xj_{\Nj}) - f(x^\star)| < \infty, \quad \E |D_{\Nj}^{(3)}| = \E \norm{\nabla f(\xj_{\Nj})}^{2} < \infty.\]
  Since $f(x^\star) > -\infty$, $\E |D_{\Nj}^{(2)}| < \infty$.
\end{proof}

\subsection{Preparation for Complexity Analysis}
Although Theorem \ref{thm:quadratic_complexity} and \ref{thm:exponential_complexity} consider the tail-randomized iterate, we start by studying two conventional output -- the randomized iterate and the last iterate. Throughout this subsection we let
\[\lamj = \etaj\mj / \bj.\]

The first lemma states a bound for expected gradient norm of the randomized iterate. 
\begin{lemma}\label{lem:randomized}
Given any positive integer $T$, let $\cR$ be a random variable supported on $\{1, \ldots, T\}$ with
 \[\P(\cR = j)\propto \lamj\]
 Then
 \[\E \norm{\nabla f(x_{\cR})}^{2}\le \frac{2\E \lb f(\tx_{0}) - f(x^\star)\rb  + \sigma^{2}\sum_{j=1}^{T}\lamj I(\Bj < n)/\Bj}{\sum_{j=1}^{T}\lamj}\]
\end{lemma}
\begin{proof}
  By Theorem \ref{thm:one_epoch},
  \[\lamj \E \norm{\nabla f(\tx_{j})}^{2}\le 2 \lb\E f(\tx_{j-1}) - \E f(\tx_{j})\rb + \frac{\sigma^{2}\lamj I(\Bj < n)}{\Bj}.\]
  By definition,
  \begin{align*}
    &\E \norm{\nabla f(x_{\cR})}^{2} = \frac{\sum_{j=1}^{T}\E \norm{\nabla f(\tx_{j})}^{2}\lamj}{\sum_{j=1}^{T}\lamj}\\
    & \le \frac{2\sum_{j=1}^{T}\lb\E f(\tx_{j-1}) - \E f(\tx_{j})\rb  + \sigma^{2}\sum_{j=1}^{T}\lamj I(\Bj < n)/\Bj}{\sum_{j=1}^{T}\lamj}\\
    & = \frac{2\lb\E f(\tx_{0}) - \E f(\tx_{T})\rb  + \sigma^{2}\sum_{j=1}^{T}\lamj I(\Bj < n)/\Bj}{\sum_{j=1}^{T}\lamj}.
  \end{align*}
The proof is then completed by the fact that $f(\tx_{T})\ge f(x^\star)$.
\end{proof}

The next lemma provides contraction results for expected gradient norm and function value suboptimality of the last iterate.
\begin{lemma}\label{lem:LjFj_recursion}
Define the following Lyapunov function
  \[\Lj = \E\lb\lamj\norm{\nabla f(\tx_{j})}^{2} + 2(f(\tx_{j}) - f(x^\star))\rb.\]
  Then under the assumption \ref{as:p-l} with $\mu$ possibly being zero,
  \begin{equation}
    \label{eq:Lj_recursion}
    \Lj\le \frac{1}{\mu\lamjj + 1}\Ljj + \frac{\sigma^{2}\lamj I(\Bj < n)}{\Bj},
  \end{equation}
  and
  \begin{equation}
    \label{eq:Fj_recursion}
    \E \lb f(\tx_{j}) - f(x^\star)\rb\le \frac{1}{\mu\lamj + 1}\E \lb f(\tx_{j-1}) - f(x^\star)\rb + \frac{\lamj}{\mu\lamj + 1}\frac{\sigma^{2}I(\Bj < n)}{2\Bj}.
  \end{equation}
\end{lemma}
\begin{proof}
  When $\mu = 0$, the lemma is a direct consequence of Theorem \ref{thm:one_epoch}. Assume $\mu > 0$ throughout the rest of the proof. Let
  \[\chij = \frac{\mu\lamj}{\mu\lamj + 1}.\]
  Then by assumption \ref{as:p-l},
  \begin{align*}
    \E (f(\tx_{j}) - f(x^\star))
    &= (1 - \chij)\E (f(\tx_{j}) - f(x^\star)) + \chij\E (f(\tx_{j}) - f(x^\star))\\
    & \le (1 - \chij)\E (f(\tx_{j}) - f(x^\star)) + \frac{\chij}{2\mu}\E\norm{\nabla f(\tx_{j})}^{2}\\
    & = \frac{1}{2(\mu\lamj + 1)}\lb\lamj\E\norm{\nabla f(\tx_{j})}^{2} + 2\E(f(\tx_{j}) - f(x^\star))\rb\\
    & = \frac{1}{2(\mu\lamj + 1)} \Lj.
  \end{align*}
  By Theorem \ref{thm:one_epoch},
  \begin{align*}
    \Lj &\le 2\E (f(\tx_{j-1}) - f(x^\star)) + \frac{\sigma^{2}\lamj I(\Bj < n)}{\Bj}\\
        &\le \frac{1}{\mu\lamjj + 1}\Ljj + \frac{\sigma^{2}\lamj I(\Bj < n)}{\Bj}.
  \end{align*}
  On the other hand, by Theorem 5,
  \[2\mu\E (f(\tx_{j}) - f(x^\star))\le \E \|\nabla f(\tx_{j})\|^{2}\le \frac{2}{\lamj}\E(f(\tx_{j-1}) - f(\tx_{j})) + \frac{\sigma^{2}I(\Bj < n)}{\Bj}.\]
  Rearranging terms concludes the proof.
\end{proof}

The third lemma shows that $\Lj$ and $\E (f(\tx_{j}) - f(x^\star))$ are uniformly bounded.
\begin{lemma}\label{lem:LjFj_uniform_bound}
 For any $j > 0$,
\[\Lj\le \DeltaL \eqdef 2\Delta_{f} + \lb\sum_{t: \Bt < n}\frac{\lamt}{\Bt}\rb\sigma^{2}.\]
\end{lemma}
\begin{proof}
  By \eqref{eq:Lj_recursion}, since $\mu \ge 0$,
  \[\Lj \le \Ljj + \frac{\sigma^{2}\lamj I(\Bj < n)}{\Bj}.\]
  Moreover, by Theorem \ref{thm:one_epoch},
  \[\Ll_{1}\le 2\E (f(\tx_{0}) - f(x^\star)) + \frac{\lambda_{1}\sigma^{2}I(B_{1} < n)}{B_{1}}.\]
  Telescoping the above inequalities yields the bound for $\Lj$. 
\end{proof}

The last lemma states refined bounds for $\E \norm{\nabla f(\tx_{j})}^{2}$ and $\E (f(\tx_{j}) - f(x^\star))$ based on Lemma \ref{lem:LjFj_recursion} and Lemma \ref{lem:LjFj_uniform_bound}.

\begin{lemma}\label{lem:LjFj_bound}
  Fix any constant $c\in (0, 1)$. Suppose $\Bj$ can be written as 
  \[\Bj = \lceil \tBj\wedge n\rceil,\]
  for some strictly increasing sequence $\tBj$. Assume that $\lamj$ is non-decreasing and
\[\frac{\tBjj\lamj}{\tBj\lamjj}\ge \sqrt{c}.\]
  Let
  \[\T_{\mu}(c) = \min\{j: \lamj > 1 / \mu c\}, \quad \T_{n} = \min\{j: \tBj \ge n\},\]
  where $\T_{\mu}(c) = \infty$ if no such $j$ exists, e.g. for $\mu = 0$. Then for any $j > \T_{\mu}(c)$, 
  \[\E \norm{\nabla f(\tx_{j})}^{2} \le \min\left\{\lb\prod_{t=\T_{\mu}(c)}^{j -1}\frac{1}{\mu\lamt}\rb\frac{\DeltaL}{\lamj} + \frac{\sigma^{2}I(j > \T_{\mu}(c))}{(1 - \sqrt{c})\tBj}, \lb\frac{1}{\mu\lambda_{\T_{n}} + 1}\rb^{(j - \T_{n})_{+}}\frac{\DeltaL}{\lamj}\right\},\]
  and
  \[\E \lb f(\tx_{j}) - f(x^\star)\rb\le \min\left\{\lb\prod_{t=\T_{\mu}(c) + 1}^{j}\frac{1}{\mu\lamt}\rb\DeltaL + \frac{\sigma^{2}I(j > \T_{\mu}(c))}{2(1 - \sqrt{c})\mu\tBj}, \lb\frac{1}{\mu\lambda_{\T_{n}} + 1}\rb^{(j - \T_{n})_{+}}\DeltaL\right\},\]
  where $\prod_{t=a}^{b}c_{t} = 1$ if $a > b$.
\end{lemma}

\begin{proof}
  We first prove the bounds involving $\T_{\mu}(c)$.  Assume $\T_{\mu}(c) < \infty$. Then for $j > \T_{\mu}(c)$,
  \begin{equation}
    \label{eq:leqc}
    \frac{1}{\mu \lamj + 1} \le \frac{1}{\mu \lamjj + 1} < \frac{1}{\mu \lamjj} < c.
  \end{equation}
  By \eqref{eq:Lj_recursion}, \eqref{eq:Fj_recursion} and the condition that $\lamj \ge \lamjj$, we have
  \[\Lj\le \frac{1}{\mu\lamjj}\Ljj + \frac{\sigma^{2}\lamj I(\Bj < n)}{\Bj}\le \frac{1}{\mu\lamjj}\Ljj + \frac{\sigma^{2}\lamj}{\tBj},\]
  and
  \begin{align*}
    \E \lb f(\tx_{j}) - f(x^\star)\rb
    &\le \frac{1}{\mu\lamj}\E \lb f(\tx_{j-1}) - f(x^\star)\rb + \frac{\sigma^{2}I(\Bj < n)}{2\mu\Bj}\\
    &\le \frac{1}{\mu\lamj}\E \lb f(\tx_{j-1}) - f(x^\star)\rb + \frac{\sigma^{2}}{2\mu\tBj}.
  \end{align*}
  Applying the above inequalities recursively and using Lemma \ref{lem:LjFj_uniform_bound} and \eqref{eq:leqc}, we obtain that
  \begin{align*}
    \lamj \E \norm{\nabla f(\tx_{j})}^{2}
    &\le \Lj\le \lb\prod_{t=\T_{\mu}(c)}^{j-1}\frac{1}{\mu\lamt}\rb\Ll_{\T_{\mu}(c)} + \sigma^{2}\sum_{t=\T_{\mu}(c) + 1}^{j}\frac{c^{j - t}\lamt}{\tBt}\\
    &\stackrel{(i)}{\le} \lb\prod_{t=\T_{\mu}(c)}^{j-1}\frac{1}{\mu\lamt}\rb\DeltaL + \sigma^{2}\sum_{t=\T_{\mu}(c) + 1}^{j}\frac{(\sqrt{c})^{j - t}\lamj}{\tBj}\\
    & = \lb\prod_{t=\T_{\mu}(c)}^{j-1}\frac{1}{\mu\lamt}\rb\DeltaL + \frac{\sigma^{2}\lamj}{(1 - \sqrt{c})\tBj}
  \end{align*}
  where (i) uses the condition that $\tBjj \lamj/ \tBj \lamjj \ge \sqrt{c}$ and thus $\Bt\ge \Bj (\sqrt{c})^{(j - t)}$. Similarly,
  \begin{align*}
    \E \lb f(\tx_{j}) - f(x^\star)\rb\le \lb\prod_{t=\T_{\mu}(c) + 1}^{j}\frac{1}{\mu\lamt}\rb\DeltaL + \frac{\sigma^{2}}{2(1 - \sqrt{c})\mu\tBj}.
  \end{align*}

  ~\\
  \noindent Next, we prove the bounds involving $\T_{n}$. Similar to the previous step, the case with $j \le \T_{n}$ can be easily proved. When $j > \T_{n}$, $\Bj = n$ and thus
  \[\Lj\le \lb \frac{1}{\mu \lambda_{\T_{n}} + 1}\rb\Ljj, \quad \E \lb f(\tx_{j}) - f(x^\star)\rb\le \lb \frac{1}{\mu \lambda_{\T_{n}} + 1}\rb\E \lb f(\tx_{j-1}) - f(x^\star)\rb.\]
  This implies the bounds involving $\T_{n}$.
\end{proof}

Combining Lemma \ref{lem:randomized} and Lemma \ref{lem:LjFj_bound}, we obtain the convergence rate of the randomized iterate. 

\begin{theorem}\label{thm:tail-randomized}
  Given any positive integer $T$, let $\cR$ be a random variable supported on $\{T, \ldots, \lceil (1 + \delta)T\rceil\}$ with
  \[\P(\cR = j)\propto \lamj.\]
  Then under the settings of Lemma \ref{lem:LjFj_bound},
  \begin{align*}
    \E \norm{\nabla f(\tx_{\cR})}^{2} \le \min\bigg\{&\lb\prod_{t=\T_{\mu}(c)}^{T-1}\frac{1}{\mu\lamt}\rb\frac{\DeltaL}{\lambda_{T}} + \frac{\sigma^{2}I(T > \T_{\mu}(c))}{(1 - \sqrt{c})\td{B}_{T}}, \lb\frac{1}{\mu\lambda_{\T_{n}} + 1}\rb^{(T - \T_{n})_{+}}\frac{\DeltaL}{\lambda_{T}}, \\
   & \quad \frac{2\DeltaL + \sigma^{2}\sum_{j=T}^{\lceil (1 + \delta)T\rceil}\lamj I(\Bj < n)/\Bj}{\sum_{j=T}^{\lceil (1 + \delta)T\rceil}\lamj}\bigg\},
  \end{align*}
  and
  \[\E \lb f(\tx_{\cR}) - f(x^\star)\rb\le \min\left\{\lb\prod_{t=\T_{\mu}(c) + 1}^{T}\frac{1}{\mu\lamt}\rb\DeltaL + \frac{\sigma^{2}I(T > \T_{\mu}(c))}{2(1 - \sqrt{c})\mu \td{B}_{T}}, \lb\frac{1}{\mu\lambda_{\T_{n}} + 1}\rb^{(T - \T_{n})_{+}}\DeltaL\right\},\]
    where $\prod_{t=a}^{b}c_{t} = 1$ if $a > b$.
\end{theorem}
\begin{proof}
  By Lemma \ref{lem:LjFj_bound}, for any $j\in [T, \lceil (1 + \delta)T\rceil]$,
  \begin{align*}
    \E \norm{\nabla f(\tx_{j})}^{2}
    &\le \min\left\{\lb\prod_{t=\T_{\mu}(c)}^{j-1}\frac{1}{\mu\lamt}\rb\frac{\DeltaL}{\lamj} + \frac{\sigma^{2}I(j > \T_{\mu}(c))}{(1 - \sqrt{c})\tBj}, \lb\frac{1}{\mu\lambda_{\T_{n}} + 1}\rb^{(j - \T_{n})_{+}}\frac{\DeltaL}{\lamj}\right\}\\
    &\le \min\left\{\lb\prod_{t=\T_{\mu}(c)}^{j-1}\frac{1}{\mu\lamt}\rb\frac{\DeltaL}{\lambda_{T}} + \frac{\sigma^{2}I(T > \T_{\mu}(c))}{(1 - \sqrt{c})\tBj}, \lb\frac{1}{\mu\lambda_{\T_{n}} + 1}\rb^{(j - \T_{n})_{+}}\frac{\DeltaL}{\lambda_{T}}\right\}.
  \end{align*}
  As a result,
  \begin{align*}
    & \E \norm{\nabla f(\tx_{\cR})}^{2} = \frac{\sum_{j=T+1}^{\lceil (1 + \delta)T\rceil}\lamj \E\norm{\nabla f(\tx_{j})}^{2}}{\sum_{j=T+1}^{\lceil (1 + \delta)T\rceil}\lamj}\\
    & \le \min\left\{\lb\prod_{t=\T_{\mu}(c)}^{T-1}\frac{1}{\mu\lamt}\rb\frac{\DeltaL}{\lambda_{T}} + \frac{\sigma^{2}I(T > \T_{\mu}(c))}{(1 - \sqrt{c})\tBj}, \lb\frac{1}{\mu\lambda_{\T_{n}} + 1}\rb^{(T - \T_{n})_{+}}\frac{\DeltaL}{\lambda_{T}}\right\}.
  \end{align*}
  Similarly we can prove the bound for $\E (f(\tx_{j}) - f(x^\star))$. To prove the third bound for $\E \norm{\nabla f(\tx_{\cR})}^{2}$, we first notice that $\tx_{\cR}$ can be regarded as the randomized iterate with $\tx_{T}$ being the initializer. By Lemma \ref{lem:randomized},
  \[\E \norm{\nabla f(\tx_{\cR})}^{2}\le \frac{2\E \lb f(\tx_{T}) - f(x^\star)\rb  + \sigma^{2}\sum_{j=T+1}^{\lceil (1 + \delta)T\rceil}\lamj I(\Bj < n)/\Bj}{\sum_{j=T+1}^{\lceil (1 + \delta)T\rceil}\lamj}.\]
  By Lemma \ref{lem:LjFj_uniform_bound},
  \[\E \lb f(\tx_{T}) - f(x^\star)\rb\le \DeltaL,\]
  which concludes the proof. The bound for $\E (f(\tx_{\cR}) - f(x^\star))$ can be proved similarly.
\end{proof}

\subsection{Complexity Analysis: Proof of Theorem \ref{thm:quadratic_complexity}}
 Under this setting,
  \[2\lamj L = \frac{2\etaj \mj}{\bj} = \sqrt{\mj} = jI(j < \sqrt{n}) + \sqrt{n}I(j\ge \sqrt{n}).\]
   Let $c = 1 / 8$. It is easy to verify that $\tBjj\lamj / \tBj\lamjj \ge 1 / 2 > \sqrt{c}$. Moreover, by Lemma \ref{lem:log1+x},
  \[L\sum_{t: \Bt < n}\frac{\lamt}{\Bt} = \sum_{t < \sqrt{n}}\frac{1}{t}\le 1 + \log\sqrt{n}.\]
Recalling the definition of $\DeltaL$ in Lemma \ref{lem:LjFj_uniform_bound},
  \begin{equation}
    \label{eq:DeltaL_bound}
    \DeltaL \le 2\Delta_{f} + \lb\sum_{t: \Bt < n}\frac{\lamt}{\Bt}\rb\sigma^{2}\le 2\lb\Delta_{f} + \frac{\sigma^{2}}{L}\log n\rb = \O(\Delta).
  \end{equation}
  Now we treat each of the three terms in the bound of $\E\norm{\nabla f(\tx_{j})}^{2}$ in Theorem \ref{thm:tail-randomized} separately. 

  ~\\
  \noindent (\textbf{First term.}) Write $\T_{\mu}$ for $\T_{\mu}(c) = \T_{\mu}(1/8)$. By definition,
\[\T_{\mu} = \min\left\{j: \lamj \ge \frac{8}{\mu}\right\} = \left\{
    \begin{array}{ll}
      \left\lceil 16L / \mu\right\rceil & \lb\left\lceil 16L / \mu \right\rceil \le \sqrt{n}\rb\\
      \infty & (\mathrm{otherwise})
    \end{array}
    \right.\]
  Let
  \[T_{g1}(\eps) = \T_{\mu} + \frac{\log(16\mu \DeltaL / \eps^{2})}{\log 8} + \frac{2\sigma}{\eps}.\]
  When $T_{g1}(\eps) = \infty$, it is obvious that $T_{g}(\eps) \le T_{g1}(\eps)$. When $T_{g1}(\eps) < \infty$, for any $T \ge T_{g1}(\eps)$,
  \[\lb\prod_{t=\T_{\mu}}^{T-1}\frac{1}{\mu\lamt}\rb \frac{\DeltaL}{\lambda_{T}}\le  c^{T - \T_{\mu}}\frac{\DeltaL}{\lambda_{T}}\le \lb\frac{1}{8}\rb^{\frac{\log(16\mu \DeltaL  / \eps^{2})}{\log 8}}\frac{\DeltaL}{\lambda_{\T_{\mu}}} = \frac{\eps^{2}}{16\mu \lambda_{\T_{\mu}}}\le \frac{\eps^{2}}{2}.\]
  Note that $\tBj = j^{2}$ in this case, 
  \[\frac{\sigma^{2}}{(1 - \sqrt{c})\td{B}_{T}}\le \frac{2\sigma^{2}}{T^{2}}\le \frac{\eps^{2}}{2}.\]
  Recalling the definition \eqref{eq:Tgeps} of $T_{g}(\eps)$, we obtain that
  \[T_{g}(\eps)\le T_{g1}(\eps).\]

  ~\\
  \noindent (\textbf{Second term.}) By definition,
\[T_{n} = \min\{j: \Bj = n\} = \lceil\sqrt{n}\rceil, \quad \lambda_{T_{n}} = T_{n} / 2L.\]
  Let
  \[T_{g2}(\eps) = T_{n} + \lb 1 + \frac{2L}{\mu\sqrt{n}}\rb\log\lb \frac{2L\DeltaL}{\sqrt{n}\eps^{2}}\rb.\]
  By Lemma \ref{lem:log1+x},
  \[T_{g2}(\eps) - T_{n}\ge \frac{\log(2L\DeltaL / \sqrt{n}\eps^{2})}{\log(1 + \mu\sqrt{n} / 2L)}.\]
  When $T \ge T_{g2}(\eps)$,
  \[\lb\frac{1}{\mu\lambda_{\T_{n}} + 1}\rb^{(T - \T_{n})_{+}}\frac{\DeltaL}{\lambda_{T}}\le \frac{\sqrt{n}\eps^{2}}{2L\DeltaL}\frac{\DeltaL}{\lambda_{T_{n}}}\le \eps^{2}.\]
  Therefore, we have
  \[T_{g}(\eps)\le T_{g2}(\eps).\]
  
  ~\\
  \noindent (\textbf{Third term.}) Note that
  \[2L\sum_{j=T}^{2T}\lamj I(\Bj < n) / \Bj = \sum_{j=T}^{2T} I(j < \sqrt{n}) / j\le \sum_{j=T}^{2T} I(j < \sqrt{n}) / T\le \frac{T + 1}{T}\le 2.\]
  and
  \[2L\sum_{j=T}^{2T}\lamj = \sum_{j=T}^{2T}\lb jI(j < \sqrt{n}) + \sqrt{n}I(j\ge \sqrt{n})\rb \ge \sum_{j=T}^{2T}(T\wedge\sqrt{n})\ge T^{2}\wedge \sqrt{n}T.\]
  Let
\[\tDeltaL = 2L\DeltaL + 2\sigma^{2}.\]
  By Theorem \ref{thm:tail-randomized},
  \[\E\norm{\nabla f(\tx_{j})}^{2}\le \frac{\tDeltaL}{T^{2}\wedge \sqrt{n}T}.\]
  Let
  \[T_{g3}(\eps) = \frac{\sqrt{\tDeltaL}}{\eps} + \frac{\tDeltaL}{\sqrt{n}\eps^{2}}.\]
  If $T \ge T_{g3}(\eps)$,
  \[\frac{\tDeltaL}{T^{2}\wedge \sqrt{n}T}\le \max\left\{\frac{\tDeltaL}{T^{2}}, \frac{\tDeltaL}{\sqrt{n}T}\right\}\le \eps^{2}, \]
  Therefore,
  \[T_{g}(\eps)\le T_{g3}(\eps).\]

  ~\\
  \noindent Putting three pieces together, we conclude that
  \[T_{g}(\eps) \le T_{g1}(\eps)\wedge T_{g2}(\eps)\wedge T_{g3}(\eps).\]
  In this case, the expected computational complexity is
  \begin{align*}
    \E \comp_{g}(\eps) &= \sum_{j=1}^{2T_{g}(\eps)}(2\mj + \Bj) = 3\sum_{j=1}^{2T_{g}(\eps)}(j^{2}\wedge n) \\
    & \le 3\min\left\{\sum_{j=1}^{2T_{g}(\eps)}j^{2}, nT_{g}(\eps)\right\} = \O\lb T_{g}^{3}(\eps) \wedge nT_{g}(\eps)\rb.
  \end{align*}

  ~\\
  \noindent \textbf{Dealing with $T_{g1}(\eps)$ and $T_{g2}(\eps)$}. First we prove that
  \begin{align}
    & \lb T_{g1}^{3}(\eps)\wedge nT_{g1}(\eps)\rb \wedge \lb T_{g2}^{3}(\eps)\wedge nT_{g2}(\eps)\rb\nonumber\\
    = & \O\lb \left\{\frac{L^{3}}{\mu^{3}} + \frac{\sigma^{3}}{\eps^{3}} + \log^{3}\lb\frac{\mu \DeltaL}{\eps^{2}}\rb\right\} \wedge \left\{n^{3/2} + \lb n + \frac{\sqrt{n}L}{\mu}\rb\log\lb\frac{L\DeltaL}{\sqrt{n}\eps^{2}}\rb\right\}\rb.\label{eq:quadratic_term12}
  \end{align}
We distinguish two cases.
\begin{itemize}
\item If $\T_{\mu}\le \T_{n}$, since $T_{g2}(\eps) > \T_{n}$ and $\T_{n}^{3}\le n\T_{n}$then
\[\lb T_{g1}^{3}(\eps)\wedge nT_{g1}(\eps)\rb\wedge \lb T_{g2}^{3}(\eps)\wedge nT_{g2}(\eps)\rb = T_{g1}^{3}(\eps),\]
which proves \eqref{eq:quadratic_term12}.
\item  If $\T_{\mu} > \T_{n}$, then
\[\lb T_{g1}^{3}(\eps)\wedge nT_{g1}(\eps)\rb\wedge \lb T_{g2}^{3}(\eps)\wedge nT_{g2}(\eps)\rb\le nT_{g2}(\eps).\]
It is left to prove
\[n^{3/2} + \lb n + \frac{\sqrt{n}L}{\mu}\rb\log\lb\frac{L\DeltaL}{\sqrt{n}\eps^{2}}\rb = \O\lb \frac{L^{3}}{\mu^{3}} + \frac{\sigma^{3}}{\eps^{3}} + \log^{3}\lb\frac{\mu \DeltaL}{\eps^{2}}\rb\rb.\]
Since $\T_{\mu} > \sqrt{n} / 2$, we have $\sqrt{n} = \O\lb \frac{L}{\mu}\rb$. This entails that
\[n^{3/2} = \O\lb\frac{L^{3}}{\mu^{3}}\rb, \quad \mbox{and}\quad n + \frac{\sqrt{n}L}{\mu}= \O\lb\frac{L^{2}}{\mu^{2}}\rb.\]
As a result,
\begin{align*}
  &\lb n + \frac{\sqrt{n}L}{\mu}\rb\log\lb\frac{L\DeltaL}{\sqrt{n}\eps^{2}}\rb = \O\lb\frac{L^{2}}{\mu^{2}}\left\{\log\lb\frac{\mu \DeltaL}{\eps^{2}}\rb + \log\lb\frac{L}{\sqrt{n}\mu}\rb\right\}\rb\\
  = & \O\lb\frac{L^{2}}{\mu^{2}}\left\{\log\lb\frac{\mu \DeltaL}{\eps^{2}}\rb + \log\lb\frac{L}{\mu}\rb\right\}\rb.
\end{align*}
\eqref{eq:quadratic_term12} is then proved by the fact that
\[\frac{L^{2}}{\mu^{2}}\log\lb\frac{\mu\DeltaL}{\eps^{2}}\rb\le \frac{L^{3}}{\mu^{3}} + \log^{3}\lb\frac{\mu\DeltaL}{\eps^{2}}\rb, \quad \frac{L^{2}}{\mu^{2}}\log\lb\frac{L}{\mu}\rb = \O\lb\frac{L^{3}}{\mu^{3}}\rb.\]
\end{itemize}

  ~\\
  \noindent \textbf{Dealing with $T_{g3}(\eps)$}. We prove that
  \begin{equation}
    \label{eq:quadratic_term3}
    T_{g3}^{3}(\eps)\wedge nT_{g3}(\eps) = \O\lb \frac{\tDeltaL^{3/2}}{\eps^{3}}\wedge \frac{\sqrt{n}\tDeltaL}{\eps^{2}}\rb.
  \end{equation}
  We distinguish two cases.
  \begin{itemize}
  \item If $\tDeltaL \le n\eps^{2}$, then
    \[\frac{\tDeltaL}{\sqrt{n}\eps^{2}}\le \frac{\sqrt{\tDeltaL}}{\eps} \le \sqrt{n}, \quad \mbox{and}\quad \frac{\tDeltaL^{3/2}}{\eps^{3}}\le \frac{\sqrt{n}\tDeltaL}{\eps^{2}}.\]
    As a result,
    \[T_{g3}(\eps) = \O\lb \frac{\sqrt{\tDeltaL}}{\eps}\rb.\]
    Thus,
    \[T_{g3}^{3}(\eps)\wedge nT_{g3}(\eps) = \O\lb T_{g3}^{3}(\eps)\rb = \O\lb \frac{\tDeltaL^{3/2}}{\eps^{3}}\rb = \O\lb \frac{\tDeltaL^{3/2}}{\eps^{3}}\wedge \frac{\sqrt{n}\tDeltaL}{\eps^{2}}\rb.\]
  \item If $\tDeltaL \ge n\eps^{2}$, then
    \[\frac{\tDeltaL}{\sqrt{n}\eps^{2}}\ge \frac{\sqrt{\tDeltaL}}{\eps} \ge \sqrt{n}, \quad \mbox{and}\quad \frac{\tDeltaL^{3/2}}{\eps^{3}}\ge \frac{\sqrt{n}\tDeltaL}{\eps^{2}}.\]
    As a result,
    \[T_{g3}(\eps) = \O\lb\frac{\tDeltaL}{\sqrt{n}\eps^{2}}\rb.\]
    Therefore,
    \[T_{g3}^{3}(\eps)\wedge nT_{g3}(\eps) = \O\lb nT_{g3}(\eps)\rb = \O\lb\frac{\sqrt{n}\tDeltaL}{\eps^{2}}\rb = \O\lb \frac{\tDeltaL^{3/2}}{\eps^{3}}\wedge \frac{\sqrt{n}\tDeltaL}{\eps^{2}}\rb.\]
  \end{itemize}
  \eqref{eq:quadratic_term3} is then proved by putting two pieces together..

  ~\\
  \noindent \textbf{Summary} Putting \eqref{eq:quadratic_term12} and \eqref{eq:quadratic_term3} together and using the fact that $\tDeltaL = \O(L\Delta)$ and $\DeltaL = \O(\Delta)$, we prove the bound for $\E \comp_{g}(\eps)$. As for $\E \comp_{f}(\eps)$, by Theorem \ref{thm:tail-randomized}, we can directly apply \eqref{eq:quadratic_term12} by replacing $\DeltaL / \lambda_{T}$ by $\DeltaL$ and $\sigma^{2}$ with $\sigma^{2} / \mu$.

  \subsection{Complexity Analysis: Proof of Theorem \ref{thm:exponential_complexity}}
 Under this setting,
  \[2\lamj L = \frac{2\etaj \mj}{\bj} = \sqrt{\mj} = \alpha^{j}I(j < \log_{\alpha}n) + \sqrt{n}I(j\ge \log_{\alpha} n).\]
  Let $c = 1 / 4\alpha^{4}$. Then
  \[\frac{\tBjj\lamj}{\tBj\lamjj} \ge \frac{1}{\alpha^{2}} > \sqrt{c}.\]
  On the other hand,
  \[L\sum_{t: \Bt < n}\frac{\lamt}{\Bt} = \frac{1}{2}\sum_{t < \sqrt{n}}\alpha^{-t}\le \frac{1}{2(1 - \alpha^{-1})} = \frac{\alpha}{2(\alpha - 1)}.\]
  Recalling the definition of $\DeltaL$ in Lemma \ref{lem:LjFj_uniform_bound},
  \begin{equation}
    \label{eq:DeltaL_bound}
    \DeltaL \le 2\Delta_{f} + \lb\sum_{t: \Bt < n}\frac{\lamt}{\Bt}\rb\sigma^{2}\le 2\Delta_{f} + \frac{\alpha}{2(\alpha - 1)}\frac{\sigma^{2}}{L} = \O(\Delta').
  \end{equation}
  As in the proof of Theorem \ref{thm:quadratic_complexity}, we treat each of the three terms in the bound of $\E\norm{\nabla f(\tx_{j})}^{2}$ in Theorem \ref{thm:tail-randomized} separately.

    ~\\
  \noindent (\textbf{First term.}) Write $\T_{\mu}$ for $\T_{\mu}(c) = \T_{\mu}(1/4\alpha^4)$. By definition,
\[\T_{\mu} = \min\left\{j: \lamj > \frac{4\alpha^2}{\mu}\right\} = \left\{
    \begin{array}{ll}
      \left\lceil \log_{\alpha} \lb 8L / \mu\rb\right\rceil + 4 & \lb\left\lceil 8L\alpha^4 / \mu \right\rceil \le \sqrt{n}\rb\\
      \infty & (\mathrm{otherwise})
    \end{array}
  \right.,\]
and
\[T_{n} = \min\{j: \Bj = n\} = \lceil(\log_{\alpha}n) / 2\rceil.\]
Let
\[A(\eps) = \max\left\{\T_{\mu} + \sqrt{2\log_{\alpha}\lb\frac{2\mu\DeltaL}{\eps^{2}}\rb}, \log_{\alpha}\lb\frac{2\sigma}{\eps}\rb\right\},\]
and
  \[T_{g1}(\eps) = A(\eps)I(A(\eps)\le \T_{n}) + \infty I(A(\eps) > \T_{n}).\]
  When $T_{g1}(\eps) = \infty$, it is obvious that $T_{g}(\eps) \le T_{g1}(\eps) = \infty$. When $T_{g1}(\eps) < \infty$, i.e. $\T_{\mu} \le A(\eps) \le \T_{n}$, for any $T \ge T_{g1}(\eps)$,
  \[\lb\prod_{t=\T_{\mu}}^{T-1}\frac{1}{\mu\lamt}\rb \frac{\DeltaL}{\lambda_{T}} = \lb\prod_{t=\T_{\mu}}^{T}\frac{1}{\mu\lamt}\rb (\mu\DeltaL) \le \exp\left\{-\sum_{t=\T_{\mu}}^{T\wedge A(\eps)}\log(\mu \lamt)\right\}(\mu \DeltaL).\]
  For any $t\in [\T_{\mu}, A(\eps)]$, since $t\le \T_{n}$, we have $\lamt = \alpha^{t}$ and thus
  \begin{align*}
    \log(\mu \lamt) = \log(\mu \lambda_{\T_{\mu}}) + (\log \alpha)(t - \T_{\mu})\ge (\log \alpha)(t - \T_{\mu}).
  \end{align*}
  Then
  \[\sum_{t=\T_{\mu}}^{T\wedge  A(\eps)}\log(\mu \lamt)\ge \frac{\log \alpha}{2}(\lfloor A(\eps)\rfloor - \T_{\mu})^{2}\ge \log\lb\frac{2\mu\DeltaL}{\eps^{2}}\rb.\]
  This implies that
  \[\lb\prod_{t=\T_{\mu}}^{T-1}\frac{1}{\mu\lamt}\rb \frac{\DeltaL}{\lambda_{T}}\le \frac{\eps^{2}}{2}.\]
  On the other hand, note that $\tBj = \alpha^{2j}$ in this case, when $T > T_{g1}(\eps)$,
  \[T_{g1}(\eps)\ge \frac{\log(2\sigma / \eps)}{\log \alpha} \Longrightarrow \frac{\sigma^{2}}{(1 - \sqrt{c})\td{B}_{T}}\le \frac{\eps^{2}}{4(1 - \sqrt{c})}\le \frac{\eps^{2}}{2},\]
  where the last inequality follows from the fact that
  \[4(1 - \sqrt{c}) = 4 \lb 1 - \frac{1}{2\alpha^{2}}\rb\ge 2.\]
  Putting pieces together we have
  \[T_{g}(\eps)\le T_{g1}(\eps).\]

    ~\\
  \noindent (\textbf{Second term.}) Let
  \[T_{g2}(\eps) = T_{n} + \lb 1 + \frac{2L}{\mu\sqrt{n}}\rb\log\lb \frac{2L\DeltaL}{\sqrt{n}\eps^{2}}\rb.\]
  Using the same argument as in the proof of Theorem \ref{thm:quadratic_complexity}, 
  \[T_{g}(\eps)\le T_{g2}(\eps).\]

    ~\\
  \noindent (\textbf{Third term.}) Note that
  \[2L\sum_{j=T}^{\lceil (1 + \delta)T \rceil}\lamj I(\Bj < n) / \Bj = \sum_{j=T}^{\lceil (1 + \delta)T \rceil} I(j < \sqrt{n}) \alpha^{-j}\le \sum_{j=1}^{\infty}\alpha^{-j} = \frac{1}{\alpha - 1}.\]
  and
  \begin{align*}
    2L\sum_{j=T}^{\lceil (1 + \delta)T \rceil}\lamj
    &= \sum_{j=T}^{\lceil (1 + \delta)T \rceil}\lb \alpha^{j}I(j < \T_{n}) + \sqrt{n}I(j\ge \T_{n})\rb \\
    &\ge \sum_{j=T+1}^{\lceil (1 + \delta)T \rceil}(\alpha^{j}\wedge\sqrt{n})\ge \alpha^{\lceil (1 + \delta)T \rceil}\wedge \delta\sqrt{n}T.
  \end{align*}
  Let
\[\tDeltaL = 2L\DeltaL + \sigma^2 / (\alpha - 1).\]
  By Theorem \ref{thm:tail-randomized},
  \[\E\norm{\nabla f(\tx_{j})}^{2}\le \frac{2\DeltaL + \sigma^2 / (\alpha - 1)}{\alpha^{\lceil (1 + \delta)T \rceil}\wedge \delta\sqrt{n}T} = \frac{\tDeltaL}{\alpha^{\lceil (1 + \delta)T \rceil}\wedge \delta\sqrt{n}T}.\]
  Let
  \[T_{g3}(\eps) = \max\left\{\frac{1}{1 + \delta}\log_{\alpha}\lb\frac{\tDeltaL}{\eps^{2}}\rb, \frac{\tDeltaL}{\delta\sqrt{n}\eps^{2}}\right\}.\]
  Then
  \[T_{g}(\eps)\le T_{g3}(\eps).\]

  ~\\
  \noindent Putting three pieces together, we conclude that
  \[T_{g}(\eps) \le T_{g1}(\eps)\wedge T_{g2}(\eps)\wedge T_{g3}(\eps).\]
  In this case, the expected computational complexity is
  \begin{align*}
    \E \comp_{g}(\eps) &= \sum_{j=1}^{\lceil (1 + \delta)T_{g}(\eps)\rceil}(2\mj + \Bj) = 3\sum_{j=1}^{\lceil (1 + \delta)T_{g}(\eps)\rceil}(\alpha^{2j}\wedge n)\\
    &\le 3\min\left\{\sum_{j=1}^{\lceil (1 + \delta)T_{g}(\eps)\rceil}\alpha^{2j}, nT_{g}(\eps)\right\} = \O\lb \alpha^{2(1 + \delta)T_{g}(\eps)} \wedge nT_{g}(\eps)\rb.
  \end{align*}

  ~\\
  \noindent \textbf{Dealing with $T_{g1}(\eps)$ and $T_{g2}(\eps)$}. First we prove that
  \begin{align}
    & \lb \alpha^{2(1 + \delta)T_{g1}(\eps)}\wedge nT_{g1}(\eps)\rb \wedge \lb \alpha^{2(1 + \delta)T_{g2}(\eps)}\wedge nT_{g2}(\eps)\rb\nonumber\\
    = & \O\lb \left\{\frac{L^{2(1 + \delta)}}{\mu^{2(1 + \delta)}}\es\lb 2\log_{\alpha}\left\{\frac{\mu\DeltaL}{\eps^{2}}\right\}\rb + \frac{\sigma^{2(1 + \delta)}}{\eps^{2(1 + \delta)}}\right\}\log^{2}n \wedge \left\{n\log \lb\frac{L}{\mu}\rb + \lb n + \frac{\sqrt{n}L}{\mu}\rb\log\lb\frac{L\DeltaL}{\sqrt{n}\eps^{2}}\rb\right\}\rb.    \label{eq:exponential_term12}
  \end{align}
We distinguish two cases.
\begin{itemize}
\item If $T_{g1}(\eps)\le \T_{n} / (1 + \delta)$, since $T_{g2}(\eps) > \T_{n}$,
  \begin{align*}
    &\lb \alpha^{2(1 + \delta)T_{g1}(\eps)}\wedge nT_{g1}(\eps)\rb\wedge \lb \alpha^{2(1 + \delta)T_{g2}(\eps)}\wedge nT_{g2}(\eps)\rb = \alpha^{2(1 + \delta)T_{g1}(\eps)}\\
    & = \O\lb\frac{L^{2(1 + \delta)}}{\mu^{2(1 + \delta)}}\es\lb 2\log_{\alpha}\left\{\frac{\mu\DeltaL}{\eps^{2}}\right\}\rb + \frac{\sigma^{2(1 + \delta)}}{\eps^{2(1 + \delta)}}\rb\\
    & = \O\lb\left\{\frac{L^{2(1 + \delta)}}{\mu^{2(1 + \delta)}}\es\lb 2\log_{\alpha}\left\{\frac{\mu\DeltaL}{\eps^{2}}\right\}\rb + \frac{\sigma^{2(1 + \delta)}}{\eps^{2(1 + \delta)}}\right\}\log^{2}n\rb,
  \end{align*}
which proves \eqref{eq:exponential_term12}.
\item  If $T_{g1}(\eps) > \T_{n} / (1 + \delta)$, then
  \begin{align*}
    &\lb \alpha^{2(1 + \delta)T_{g1}(\eps)}\wedge nT_{g1}(\eps)\rb \wedge \lb \alpha^{2(1 + \delta)T_{g2}(\eps)}\wedge nT_{g2}(\eps)\rb \le nT_{g2}(\eps)\\
    & \,\, = \O\lb n\log \lb\frac{L}{\mu}\rb + \lb n + \frac{\sqrt{n}L}{\mu}\rb\log\lb\frac{L\DeltaL}{\sqrt{n}\eps^{2}}\rb\rb.
  \end{align*}
It is left to prove that
\begin{align}
  & n\log \lb\frac{L}{\mu}\rb + \lb n + \frac{\sqrt{n}L}{\mu}\rb\log\lb\frac{L\DeltaL}{\sqrt{n}\eps^{2}}\rb\nonumber\\
  = & \O\lb\left\{\frac{L^{2(1 + \delta)}}{\mu^{2(1 + \delta)}}\es\lb 2\log_{\alpha}\left\{\frac{\mu\DeltaL}{\eps^{2}}\right\}\rb + \frac{\sigma^{2(1 + \delta)}}{\eps^{2(1 + \delta)}}\right\}\log^{2}n\rb.\label{eq:exponential_term12_goal}
\end{align}
We consider the following two cases.
\begin{itemize}
\item If $L / \mu > \sqrt{n}$,
  \begin{align*}
    &n\log\lb\frac{L}{\mu}\rb + \lb n + \frac{\sqrt{n}L}{\mu}\rb\log\lb\frac{L\DeltaL}{\sqrt{n}\eps^{2}}\rb \le \frac{\sqrt{n}L}{\mu}\log\lb\frac{L}{\mu}\rb + \frac{2\sqrt{n}L}{\mu}\log\lb\frac{L\DeltaL}{\sqrt{n}\eps^{2}}\rb\\
    & = \O\lb\frac{\sqrt{n}L}{\mu}\log\left\{\frac{L^{2}\DeltaL}{\mu \sqrt{n}\eps^{2}}\right\}\rb = \O\lb\frac{\sqrt{n}L}{\mu}\log\lb\frac{\mu\DeltaL}{\eps^{2}}\rb + \frac{\sqrt{n}L}{\mu}\log\lb\frac{1}{\sqrt{nL^{2}}\mu^{2}}\rb\rb
  \end{align*}
  The first term can be bounded by
  \[\frac{\sqrt{n}L}{\mu}\log\lb\frac{\mu\DeltaL}{\eps^{2}}\rb\le \frac{L^{2}}{\mu^{2}}\log\lb\frac{\mu\DeltaL}{\eps^{2}}\rb = \O\lb\frac{L^{2(1 + \delta)}}{\mu^{2(1 + \delta)}}\es\lb 2\log_{\alpha}\left\{\frac{\mu\DeltaL}{\eps^{2}}\right\}\rb\log^{2} n\rb.\]
  To bound the second term, we consider two cases.
  \begin{itemize}
  \item If $L / \mu > n$,
    \begin{align*}
      & \frac{\sqrt{n}L}{\mu}\log\lb\frac{L^{2}}{\sqrt{n}\mu^{2}}\rb\le \frac{2L^{3/2}}{\mu^{3/2}}\log\lb\frac{L}{\mu}\rb = \O\lb\frac{L^{2}}{\mu^{2}}\rb\\
      & \,\, = \O\lb\frac{L^{2(1 + \delta)}}{\mu^{2(1 + \delta)}}\es\lb 2\log_{\alpha}\left\{\frac{\mu\DeltaL}{\eps^{2}}\right\}\rb\log^{2} n\rb.
    \end{align*}
  \item If $\sqrt{n} < L / \mu < n$,
  \[\frac{\sqrt{n}L}{\mu}\log\lb\frac{L^{2}}{\sqrt{n}\mu^{2}}\rb\le \frac{2L\sqrt{n}\log n}{\mu} \le \frac{2L^{2}\log n}{\mu^{2}} = \O\lb\frac{L^{2(1 + \delta)}}{\mu^{2(1 + \delta)}}\es\lb 2\log_{\alpha}\left\{\frac{\mu\DeltaL}{\eps^{2}}\right\}\rb\log^{2} n\rb.\]
  \end{itemize}
  \eqref{eq:exponential_term12_goal} is then proved by putting pieces together.
\item If $L / \mu \le \sqrt{n}$.
  \begin{align*}
    &n\log \lb\frac{L}{\mu}\rb + \lb n + \frac{\sqrt{n}L}{\mu}\rb\log\lb\frac{L\DeltaL}{\sqrt{n}\eps^{2}}\rb \le n\log n + 2n\log\lb\frac{L\DeltaL}{\sqrt{n}\eps^{2}}\rb\\
    & = n\log n + 2n\log \lb\frac{L}{\mu \sqrt{n}}\rb + 2n\log\lb\frac{\mu\DeltaL}{\eps^{2}}\rb = \O\lb n\log n + n \log_{\alpha}\lb\frac{\mu\DeltaL}{\eps^{2}}\rb\rb
  \end{align*}
  Since $T_{g1}(\eps) > \T_{n} / (1 + \delta)$,
  \begin{equation}
    \label{eq:exponential_term12_nlower}
    n\le \alpha^{2(1 + \delta)T_{g1}(\eps)} = \O\lb\frac{L^{2(1 + \delta)}}{\mu^{2(1 + \delta)}}\es\lb 2\log_{\alpha}\left\{\frac{\mu\DeltaL}{\eps^{2}}\right\}\rb + \frac{\sigma^{2(1 + \delta)}}{\eps^{2(1 + \delta)}}\rb.
  \end{equation}
  It is left to prove that
  \begin{equation}
    \label{eq:exponential_term12_goal2}
    \frac{n}{\log^{2}n}\log_{\alpha}\lb\frac{\mu\DeltaL}{\eps^{2}}\rb = \O\lb\frac{L^{2(1 + \delta)}}{\mu^{2(1 + \delta)}}\es\lb 2\log_{\alpha}\left\{\frac{\mu\DeltaL}{\eps^{2}}\right\}\rb + \frac{\sigma^{2(1 + \delta)}}{\eps^{2(1 + \delta)}}\rb.
  \end{equation}
  We distinguish two cases.
  \begin{itemize}
  \item If $\log_{\alpha}(\mu\DeltaL / \eps^{2}) \le 2\log^{2}n$, \eqref{eq:exponential_term12_goal2} is proved by \eqref{eq:exponential_term12_nlower}.
  \item If $\log_{\alpha}(\mu\DeltaL / \eps^{2}) > 2\log^{2}n$,
    \[\es\lb 2\log_{\alpha}\left\{\frac{\mu \DeltaL}{\eps^{2}}\right\}\rb\ge \es\lb\log_{\alpha}\left\{\frac{\mu \DeltaL}{\eps^{2}}\right\} / 2\rb^{2}\ge n \cdot\es\lb \log_{\alpha}\left\{\frac{\mu \DeltaL}{\eps^{2}}\right\} / 2\rb.\]
    Note that
    \[\log_{\alpha}\lb\frac{\mu\DeltaL}{\eps^{2}}\rb = \O\lb\es\lb\log_{\alpha}\left\{\frac{\mu \DeltaL}{\eps^{2}}\right\} / 2\rb\rb.\]
    Therefore, 
    \[ \frac{n}{\log^{2}n}\log_{\alpha}\left\{\frac{\mu\DeltaL}{\eps^{2}}\right\} = \O\lb \es\lb 2\log_{\alpha}\left\{\frac{\mu \DeltaL}{\eps^{2}}\right\}\rb\rb,\]
    which proves \eqref{eq:exponential_term12_goal2}.
  \end{itemize}
\end{itemize}
\end{itemize}
Therefore, \eqref{eq:exponential_term12_goal} is proved. 

~\\
\noindent \textbf{Dealing with $T_{g3}(\eps)$}. If $\delta =
0$, the bound is infinite and thus trivial. Assume $\delta > 0$. We prove that
\begin{equation}
  \label{eq:exponential_term3}
  \alpha^{2(1 + \delta)T_{g3}(\eps)}\wedge nT_{g3}(\eps) = \O\lb \frac{\tDeltaL^{2}}{\delta^{2}\eps^{4}}\wedge \frac{\sqrt{n}\tDeltaL\log n}{\delta \eps^{2}}\rb.
\end{equation}
Let
\[h(y) = \frac{1}{1 + \delta}\log_{\alpha}(y) - \frac{y}{\delta\sqrt{n}}.\]
It is easy to see that
\[h'(y) = \frac{1}{y(1 + \delta)\log \alpha} - \frac{1}{\delta \sqrt{n}}.\]
Thus $h(y)$ is decreasing on $[0, y^{*}]$ and increasing on $[y^{*}, \infty)$ where
\[y^{*} = \frac{\delta \sqrt{n}}{(1 + \delta)\sqrt{\alpha}}.\]
Now we distinguish two cases.
\begin{itemize}
\item If $h(y^{*})\le 0$, then $h(y)\le 0$ for all $y > 0$ and thus $h(\tDeltaL / \eps^{2})\le 0$. As a result,
  \[h\lb\frac{\tDeltaL}{\eps^{2}}\rb \le 0 \Longrightarrow T_{g3}(\eps) \le \frac{\tDeltaL}{\delta\sqrt{n}\eps^{2}}.\]
  If $\tDeltaL/\delta\eps^{2} \le \sqrt{n}$,
  \[T_{g3}(\eps) = \O(1)\Longrightarrow\alpha^{2(1 + \delta)T_{g3}(\eps)}\wedge nT_{g3}(\eps) = \O(1),\]
  and hence \eqref{eq:exponential_term3} is proved by recalling the footnote in page \pageref{fn:O}. Otherwise,   note that
    \[\alpha^{2(1 + \delta)T_{g3}(\eps)}\wedge nT_{g3}(\eps) = \O(nT_{g3}(\eps)) = \O\lb\frac{\sqrt{n}\tDeltaL}{\delta\eps^{2}}\rb.\]
  Since $\tDeltaL/\delta\eps^{2} > \sqrt{n}$,
  \[\frac{\sqrt{n}\tDeltaL}{\delta\eps^{2}} = \O\lb\frac{\tDeltaL^{2}}{\delta^{2}\eps^{4}}\rb.\]
  Therefore, \eqref{eq:exponential_term3} is proved.
\item If $h(y^{*}) > 0$, noting that $h(0) = h(\infty) = -\infty$, there must exist $0 < y_{1}^{*} < y^{*} < y_{2}^{*} < \infty$ such that $h(y_{1}^{*}) = h(y_{2}^{*}) = 0$ and $h(y) \ge 0$ iff $y \in [y_{1}^{*}, y_{2}^{*}]$. 
  First we prove that
  \begin{equation}
    \label{eq:ystar}
    y_{1}^{*} = \O\lb 1\rb, \quad y_{2}^{*} = \O(\delta\sqrt{n}\log n).
  \end{equation}
  As for $y_{1}^{*}$, if $y^{*}\le 4$, then $y_{1}^{*} \le y^{*} = \O(1)$. If $y^{*} > 4$, let
  \[y = 1 + \frac{4}{y^{*}}.\]
  Now we prove $y_{1}^{*}\le y$. It is sufficient to prove $h(y)\ge 0$ and $y \le y^{*}$. In fact, a simple algebra shows that
  \[h(y^{*})\ge 0\Longrightarrow y^{*}\ge e\Longrightarrow y \le 4 \le y^{*}.\]
  On the other hand, by Lemma \ref{lem:log1+x}
  \[\log y\ge \frac{4/y^{*}}{1 + 4/y^{*}}\ge \frac{2}{y^{*}}.\]
  Recalling that $y^{*} = \delta\sqrt{n} / (1 + \delta)\log \alpha$, 
  \[h(y) \ge \frac{2}{(1 + \delta)(\log \alpha)y^{*}} - \frac{1}{\delta\sqrt{n}}\lb 1 + \frac{4}{y^{*}}\rb = \frac{1}{\delta\sqrt{n}}\lb 2 - 1 - \frac{4}{y^{*}}\rb\ge 0.\]
  Therefore, $y_{1}^{*} = \O(1)$.

  ~\\
  \noindent As for $y_{2}^{*}$, let $C > 0$ be any constant, then for sufficiently large $C$,
  \[(C + 1)\delta\sqrt{n}\log_{\alpha}n\ge y^{*} = \frac{\delta\sqrt{n}}{(1 + \delta)\log \alpha}.\]
  On the other hand,
  \[h((C + 1)\delta\sqrt{n}\log_{\alpha}(n)) = \log_{\alpha}(C\log_{\alpha} n) - C\log_{\alpha}(\delta\sqrt{n}).\]
Then for sufficiently large $C$,
  \[h((C + 1)\delta\sqrt{n}\log_{\alpha}(n))\le 0.\]
  Recalling that $h(y)$ is decreasing on $[y^{*}, \infty)$ and $h(y_{2}^{*}) = 0$, \eqref{eq:ystar} must hold. Based on \eqref{eq:ystar}, \eqref{eq:exponential_term3} can be equivalently formulated as
  \begin{equation}
    \label{eq:exponential_term3_goal2}
    \alpha^{2(1 + \delta)T_{g3}(\eps)}\wedge nT_{g3}(\eps) = \O\lb \frac{\tDeltaL^{2}}{\delta^{2}\eps^{2}}\left\{\frac{\tDeltaL^{2}}{\eps^{2}}\wedge y_{2}^{*}\right\}\rb.
  \end{equation}
  Now we consider three cases.
  \begin{itemize}
  \item If $\tDeltaL / \eps^{2}\ge y_{2}^{*}$,
    \[h\lb\frac{\tDeltaL}{\eps^{2}}\rb\le 0 \Longrightarrow T_{g3}(\eps) = \frac{\tDeltaL}{\delta\sqrt{n}\eps^{2}}.\]
    Then
    \[\alpha^{2(1 + \delta)T_{g3}(\eps)}\wedge nT_{g3}(\eps) = \O(nT_{g3}(\eps)) = \O\lb\frac{\sqrt{n}\tDeltaL}{\delta \eps^{2}}\rb = \O\lb\frac{\tDeltaL}{\delta^{2} \eps^{2}}y_{2}^{*}\rb\]
    where the last equality uses the fact that
\[y_{2}^{*}\ge y^{*} = \frac{\delta\sqrt{n}}{(1 + \delta)\log \alpha}.\]
This proves \eqref{eq:exponential_term3_goal2}.
 \item If $\tDeltaL / \eps^{2}\le y_{1}^{*}$,
   \[h\lb\frac{\tDeltaL}{\eps^{2}}\rb\le 0 \Longrightarrow T_{g3}(\eps) = \frac{\tDeltaL}{\delta\sqrt{n}\eps^{2}}.\]
   By \eqref{eq:ystar},
   \[T_{g3}(\eps) = \O(1) \Longrightarrow\alpha^{2(1 + \delta)T_{g3}(\eps)}\wedge nT_{g3}(\eps) = \O(1),\]
   and hence \eqref{eq:exponential_term3} is proved by recalling the footnote in page \pageref{fn:O}.
 \item If $\tDeltaL / \eps^{2}\in [y_{1}^{*}, y_{2}^{*}]$,
   \[h\lb\frac{\tDeltaL}{\eps^{2}}\rb\ge 0 \Longrightarrow T_{g3}(\eps) = \frac{1}{1 + \delta}\log_{\alpha}\lb\frac{\tDeltaL}{\eps^{2}}\rb.\]
   Then
   \[\alpha^{2(1 + \delta)T_{g3}(\eps)}\wedge nT_{g3}(\eps) = \O\lb\alpha^{2(1 + \delta)T_{g3}(\eps)}\rb = \O\lb\frac{\tDeltaL^{2}}{\eps^{4}}\rb = \O\lb\frac{\tDeltaL}{\eps^{2}}\left\{\frac{\tDeltaL}{\eps^{2}}\wedge y_{2}^{*}\right\}\rb,\]
   which proves \eqref{eq:exponential_term3_goal2} since $\delta = O(1)$.
 \end{itemize}
\end{itemize}

~\\
\noindent \textbf{Summary} Putting \eqref{eq:exponential_term12} and \eqref{eq:exponential_term3} together and using the fact that $\tDeltaL = \O(L\Delta')$ and $\DeltaL = \O(\Delta')$, we prove the bound for $\E \comp_{g}(\eps)$. As for $\E \comp_{f}(\eps)$, by Theorem \ref{thm:tail-randomized}, we can directly apply \eqref{eq:exponential_term12} by replacing $\DeltaL / \lambda_{T}$ by $\DeltaL$ and $\sigma^{2}$ with $\sigma^{2} / \mu$.

\subsection{Complexity analysis: Proof of Theorem \ref{thm:non-adaptive}}



For the first claim, we set $\delta = 0$ thus $\cR(T) = T$. Applying \eqref{eq:Lj_recursion} recursively with the fact $\sum_{i=0}^T 1/(1+x)^i \leq \nicefrac{(1+x)}{x}$ for $x>0$, we obtain
\[
    \E \lb f(\tx_{\cR(T)}) - f(x^\star) \rb\le \frac{1}{\lb\mu\lambda + 1\rb^T}\E \lb f(\tx_{0}) - f(x^\star)\rb +\frac{\sigma^{2}I(B < n)}{2\mu B},
\]
where $\lambda = \nicefrac{\sqrt{B}}{2L}$.  Setting $B =  \lb n \wedge \frac{\sigma^2}{4\mu\eps^2} \rb$, the second term is less than $\nicefrac{\eps^2}{2}$. For $T \geq \lb 1+ \nicefrac{2L}{\mu \sqrt{B}}\rb\log\frac{2\Delta_f}{\eps^2} $ also the first term is less han $\nicefrac{\eps^2}{2}$ which follows from Lemma~\ref{lem:log1+x}. As the cost of each epoch is $2B$ this result implies that the total complexity is 
\[ 
\O\lb   \lb  B + \frac{\sqrt{B}L}{\mu}\rb\log\frac{\Delta_f}{\eps^2} \rb.
\]

For the second claim, we use \eqref{eq:Lj_recursion} recursively together with Theorem~\ref{thm:one_epoch} and with the fact $\sum_{i=0}^{\infty} 1/(1+x)^i \leq \nicefrac{(1+x)}{x}$ for $x>0$, we obtain 
\[
\Ll_T\le \frac{\Ll_{1}}{(\mu\lambda + 1)^{T-1}} + \frac{\sigma^{2}\lambda I(B < n)}{B}\frac{1 + \lambda\mu}{\lambda\mu}.
\]
Further by Theorem \ref{thm:one_epoch},
\[\Ll_{1}\le 2\Delta_{f} + \frac{\sigma^{2}\lambda I(B < n)}{B}.\]
Using definition of $\Lj$ and $\delta = 0$, we get

\begin{align*}
  \E \|\nabla f(\tx_{\cR(T)})\|^{2} &\le \frac{2\Delta_f}{\lambda(\mu\lambda + 1)^{T-1}} + \frac{\sigma^{2} I(B < n)}{B}\lb 2 + \frac{1}{\lambda\mu} \rb\\
  & \le \frac{2\Delta_f}{\lambda(\mu\lambda + 1)^{T-1}} + \frac{\sigma^{2} I(B < n)}{B}\lb 2 + \frac{2L}{\sqrt{B}\mu} \rb
\end{align*}

The choice of $B$ to be $ \left( \left\{\frac{8\sigma^{2}}{\eps^{2}} + \frac{8\sigma^{4/3}L^{2/3}}{\eps^{4/3}\mu^{2/3}}\right\} \wedge n \right)$
guarantees that the second term is less than $\nicefrac{\eps^2}{2}$. By the same reasoning as for the second claim, we obtain following complexity

\[ 
\O\lb   \lb  B +  \frac{\sqrt{B}L}{\mu}\rb\log\frac{L\Delta_f}{\sqrt{B}\eps^2} \rb.
\]

\section{Miscellaneous}

\begin{lemma}\label{lem:var_sampling}
Let $z_{1}, \ldots, z_{M}\in \R^{d}$ be an arbitrary population and $\mathcal{J}$ be a uniform random subset of $[M]$ with size $m$. Then 
\[\Var\lb\frac{1}{m}\sum_{j\in \mathcal{J}}z_{j}\rb\le \frac{I(m < M)}{m}\cdot \frac{1}{M}\sum_{j=1}^{M}\|z_{j}\|_{2}^{2}.\]
\end{lemma}

\begin{lemma}\label{lem:sum1/t}
  For any positive integer $n$,
  \[\sum_{t=1}^{n}\frac{1}{t}\le 1 + \log n.\]
\end{lemma}
\begin{proof}
  Since $x\mapsto 1/x$ is decreasing,
  \[\sum_{t=1}^{n}\frac{1}{t} = 1 + \sum_{t=2}^{n}\frac{1}{t}\le 1 + \int_{1}^{n}\frac{dx}{x} = 1 + \log n.\]
\end{proof}

\begin{lemma}\label{lem:log1+x}
  For any $x > 0$,
  \[\frac{1}{\log(1 + x)}\le 1 + \frac{1}{x}.\]
\end{lemma}
\begin{proof}
  Let $g(x) = (1 + x)\log (1 + x) - x$. Then
  \[g'(x) = 1 + \log(1 + x) - 1 = \log(1 + x)\ge 0.\]
  Thus $g$ is increasing on $[0, \infty)$. As a result, $g(x)\ge g(0) = 0$. 
\end{proof}

\end{document}